\documentclass{article}
\usepackage{iclr2027_conference,times}

\usepackage{amsmath,amsfonts,bm}

\def\eqref#1{equation~\ref{#1}}

\def\1{\bm{1}}

\DeclareMathAlphabet{\mathsfit}{\encodingdefault}{\sfdefault}{m}{sl}
\SetMathAlphabet{\mathsfit}{bold}{\encodingdefault}{\sfdefault}{bx}{n}

\usepackage{hyperref}
\usepackage{url}
\usepackage{booktabs}
\usepackage{graphicx}
\usepackage{amsmath}
\usepackage{amssymb}
\usepackage[table]{xcolor}
\usepackage{array}
\usepackage{multirow}
\usepackage{subcaption}
\usepackage{enumitem}
\usepackage{wrapfig}
\usepackage{algorithm}
\usepackage{algpseudocode}
\usepackage{pifont}
\usepackage{amsthm}
\newtheorem{proposition}{Proposition}

\definecolor{hred}{RGB}{168,26,38}
\definecolor{hblue}{RGB}{26,62,148}
\newcommand{\hi}[1]{\textbf{\textcolor{hred}{#1}}}

\newcommand{\cthead}[1]{\begin{tabular}[c]{@{}c@{}}#1\end{tabular}}

\newcommand{\sd}[1]{\underline{\textcolor{hblue}{#1}}}

\newcommand{\method}{\textsc{OneShot}}
\newcommand{\kshot}{$k$-\textsc{Shot}}

\title{\method{}: Label-Efficient Detector Weighting for
Time Series Anomaly Detection}

\author{Seunghan Lee, Jun Seo, Jaehoon Lee, Junhyeok Kang, Sangjun Han, Sungdong Yoo, \\
\bf Minjae Kim, Tae Yoon Lim, Dongwan Kang, Hwanil Choi, Soonyoung Lee, Wonbin Ahn \\
LG AI Research}

\iclrfinalcopy

\newcommand{\nseries}{152}
\newcommand{\ndropped}{7}

\newcommand{\ntemps}{8}
\newcommand{\taulist}{0.01, 0.02, 0.035, 0.05, 0.075, 0.1, 0.15 and 0.3}

\newcommand{\taupick}{0.02}

\IfFileExists{tables/macros_diag.tex}{
\newcommand{\neffmed}{6.1}
\newcommand{\ndetmed}{14}

}{}
\IfFileExists{tables/macros_pop.tex}{\newcommand{\ntotal}{350}
\newcommand{\neligible}{159}
\newcommand{\nsingle}{154}
\newcommand{\ndegenerate}{15}
\newcommand{\nshort}{22}
}{}
\IfFileExists{tables/macros_pool.tex}{
}{}
\IfFileExists{tables/macros_first.tex}{
}{}
\IfFileExists{tables/macros_noisebig.tex}{

}{}
\IfFileExists{tables/macros_mvdiag.tex}{\newcommand{\nmvdiag}{73}
\newcommand{\mvneff}{2.6}

\newcommand{\mvndet}{9}
}{}
\IfFileExists{tables/macros_det.tex}{

}{}
\IfFileExists{tables/macros_review.tex}{\newcommand{\ndrop}{521}
\newcommand{\ntotfile}{2584}
\newcommand{\droppct}{20.2}
\newcommand{\medkept}{14}
\newcommand{\medraw}{17}

\newcommand{\ndropbeat}{2}

\newcommand{\causargmax}{-0.016}
\newcommand{\causpargmax}{0.04}
\newcommand{\caussoft}{-0.006}
\newcommand{\causpsoft}{0.82}

\newcommand{\sdsoft}{0.039}
\newcommand{\sdargmax}{0.058}
}{}

\IfFileExists{tables/macros_cost.tex}{
\newcommand{\ncost}{24}
\newcommand{\costms}{45}
\newcommand{\costpool}{44.9}

}{}
\providecommand{\costms}{47}
\providecommand{\costpool}{41}

\IfFileExists{tables/macros_nshot.tex}{

}{}
\IfFileExists{tables/macros_cleanneg.tex}{
\newcommand{\negcontam}{6.0\%}
\newcommand{\negdrho}{-0.7}
}{}
\IfFileExists{tables/macros_metrics.tex}{
\newcommand{\nmvmetricseries}{73}
}{}
\IfFileExists{tables/macros_win.tex}{\newcommand{\winlow}{51}

\newcommand{\nwinseries}{152}
}{}
\IfFileExists{tables/macros_bestk.tex}{\newcommand{\bestkmode}{2}
\newcommand{\bestkpct}{36.2}

}{}
\IfFileExists{tables/macros_negsize.tex}{

\newcommand{\negnseries}{150}

}{}
\IfFileExists{tables/macros_neffk.tex}{

}{}
\IfFileExists{tables/macros_mvfam.tex}{\newcommand{\mvfamnseries}{55}

}{}
\IfFileExists{tables/macros_mvn.tex}{
}{}
\IfFileExists{tables/macros_pools.tex}{\newcommand{\nwinlen}{9}
\newcommand{\winlist}{15, 25, 35, 50, 75, 150, 200, 300 and 600}
\newcommand{\nwindet}{4}
\newcommand{\windets}{HSF, HSF\_Causal, HSF\_U, PaAno\_PAI}

}{}
\IfFileExists{tables/macros_pool2.tex}{
}{}
\IfFileExists{tables/macros_data.tex}{\newcommand{\nsrcu}{19}
\newcommand{\nsrcmv}{14}
\newcommand{\nsrcall}{23}
}{}
\IfFileExists{tables/macros_noiseret.tex}{

}{}

\begin{document}

\maketitle
\lhead{Preprint.}

\vspace{-2pt}
\begin{abstract}
\vspace{-2pt}
Time series (TS) anomaly detection (AD) lacks a \textit{universal detector}, with the
best choice determined by the series rather than by the algorithm. Existing methods
therefore either select one detector for each series or average a pool of them.
However, both assume that no label ever arrives, while a deployed monitor does receive
them, since an operator confirms the alarms it raises. In this paper, we propose
\method{}, which uses a single confirmed anomaly to weight the pool rather than to
choose one detector from it. Specifically, we score each detector on the confirmed
anomaly, and those scores become the weights of a soft mixture over the whole pool. We
show that \textit{a single confirmed anomaly is enough} to come close to the best
detector for the series. Across \nseries{} univariate and \nmvmetricseries{} multivariate
series, \method{} outperforms the strongest single detector, averaging the whole pool,
and even picking the best detector on the confirmed event. Code is available at:
\url{https://github.com/seunghan96/one_shot}.
\end{abstract}

\vspace{-2pt}
\section{Introduction}
\label{sec:intro}
\vspace{-2pt}

\IfFileExists{figs/fig_wins.pdf}{%
\begin{wrapfigure}{r}{0.34\textwidth}
\vspace{-8pt}
\centering
\includegraphics[width=0.30\textwidth]{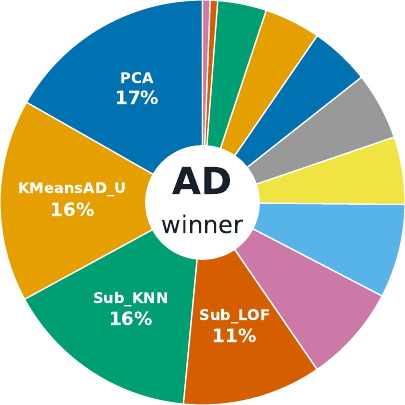}
\caption{\textbf{No universal detector.} How often each detector wins.}
\label{fig:wins}
\vspace{-16pt}
\end{wrapfigure}}{}
Anomaly detection (AD) in time series (TS) supports monitoring in industrial control
\citep{mathur2016swat}, health care \citep{moody2001impact}, web services
\citep{ren2019time}, and finance \citep{lee2026exaone}. Recent benchmarks evaluate dozens of detectors and
report that \textit{no single detector wins across datasets}
\citep{liu2024elephant,paparrizos2022tsb}. As shown in Figure~\ref{fig:wins}, the
wins are spread over the pool rather than held by one detector. A practitioner who
deploys a monitor therefore faces a choice among various candidates, and that
choice matters more than the design of any one candidate.

An operator who answers a running monitor's alarms confirms events as they arrive, and
each confirmation is a label on the series itself. However, existing solutions
assume that no labels ever arrive \citep{zhao2021automating,zhao2022toward}. That
leaves them ranking detectors by the scores or by results on other datasets
\citep{marques2015internal}. A monitor starts from one confirmed event,
and we ask how close that brings a deployment to the best detector.

\IfFileExists{figs/fig2_final.pdf}{%
\begin{wrapfigure}{r}{0.46\textwidth}
\vspace{-14pt}
\centering
\includegraphics[width=0.44\textwidth]{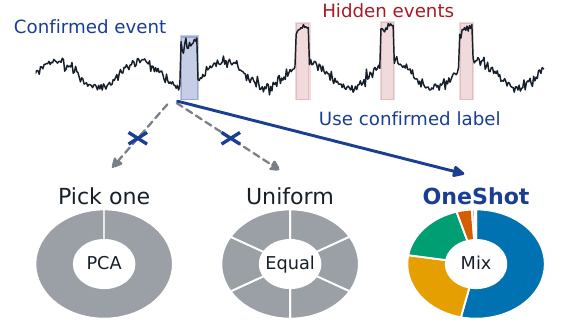}
\vspace{-4pt}
\caption{\textbf{Prior works vs.\ \method{}.} Only \method{} weights every detector
by how well it did on the confirmed event.}
\label{fig:online}
\vspace{-10pt}
\end{wrapfigure}}{}
In this paper, we show that confirming a \textit{single anomaly event is enough} to close most
of the gap to the best detector for the series. As shown in Figure~\ref{fig:online}, the method
spends that one event on a weight over the whole pool, while the alternatives ignore it. We
measure this gap on \nseries{} univariate series and 18 detectors from a curated benchmark
\citep{liu2024elephant}, under the measure it reports as most reliable
\citep{paparrizos2022volume}. This holds for an event drawn at random, i.e., an operator
need not pick a good one. Note that we never retrain a detector, i.e., the label changes how the detectors are weighted rather than the detectors themselves.

Motivated by this observation, we propose \method{}, which replaces the hard choice of
one detector with a soft mixture whose weights come from the confirmed event.
Specifically, we score each detector on that window and turn those scores
into a softmax weight. The combined score then follows the detectors that did well on
the window, while the rest keep a smaller weight. We set the temperature that decides
how sharply the weights follow the window on held-out series, and the accuracy stays
flat over a wide band of values. The main contributions are summarized as follows:
\begin{itemize}[leftmargin=1.2em, itemsep=1pt, topsep=2pt, parsep=0pt]
  \item We show that a single confirmed anomaly event closes most of the gap
        between a detector chosen without labels and the best detector for the series,
        whether the event is the first one a monitor confirms or any other.
  \item We propose \method{}, a deliberately simple softmax weighting over a detector pool
        without retraining, which isolates what one label is worth. It beats the window
        winner and unsupervised ensembles, and matches the best fixed top-$K$ average,
        whose $K$ is unknown in advance.
  \item We propose a \kshot{} extension that reads every event confirmed so far
        through the same softmax, which lets a deployment spend the labels it
        accumulates without retuning anything.
  \item We keep the cost of the method at one pass over that window and a
        weighted sum of the scores the pool already produces, which adds no detector
        run to a deployed pipeline.
  \item We conduct extensive experiments on \nseries{} univariate series and
        \nmvdiag{} multivariate series, with three detector pools of up to 77
        detectors, demonstrating the effectiveness of \method{}.
\end{itemize}

\section{Related Work}
\label{sec:related}

\providecommand{\yes}{\textcolor{hred}{\ding{51}}}
\providecommand{\no}{\textcolor{black!35}{\ding{55}}}
\begin{table}[t]
\centering
\caption{\textbf{Comparison of AD detector families.}
\textit{Frozen}
means no training loop, \textit{No corpus} no pool of past evaluations, and
\textit{Label} ground truth from the monitored series. \textit{Mixture} means
every detector keeps a share of the weight, and \textit{Threshold} that the row also
says where to raise the alarm.}
\label{tab:families}
\footnotesize
\setlength{\tabcolsep}{4.5pt}
\begin{tabular}{@{}l|ccc|cc@{}}
\toprule
\multirow{2.5}{*}{Family} & \multicolumn{3}{c|}{\textit{What it needs}} & \multicolumn{2}{c@{}}{\textit{What it returns}} \\
\cmidrule(lr){2-4}\cmidrule(l){5-6}
 & Frozen & No corpus & Label & Mixture & Threshold \\
\midrule
\multicolumn{6}{@{}l}{\textit{Single detector (w/o labels)}} \\
\quad One fixed detector \citep{liu2024elephant}                 & \yes & \yes & \no  & \no  & \no \\
\quad Internal measures \citep{ma2023need}                       & \yes & \yes & \no  & \no  & \no \\
\quad Local selection \citep{zhao2019lscp}                       & \yes & \yes & \no  & \no  & \no \\
\quad Meta-learned selection \citep{zhao2021automating}          & \yes & \no  & \no  & \no  & \no \\
\quad Classifier-based selection \citep{sylligardos2023choose}   & \yes & \no  & \no  & \no  & \no \\
\midrule
\multicolumn{6}{@{}l}{\textit{Detector ensemble (w/o labels)}} \\
\quad Score averaging \citep{aggarwal2015theoretical}            & \yes & \yes & \no  & \yes & \no \\
\quad Feature bagging \citep{lazarevic2005feature}               & \no  & \yes & \no  & \yes & \no \\
\quad Selective ensembles \citep{rayana2016less}                 & \yes & \yes & \no  & \no  & \no \\
\midrule
\multicolumn{6}{@{}l}{\textit{Trained detector (w/ labels)}} \\
\quad Semi-supervised detectors \citep{ruff2020deep}             & \no  & \yes & \yes & \no  & \no \\
\quad Active discovery \citep{das2016incorporating}              & \no  & \yes & \yes & \no  & \no \\
\midrule
\rowcolor{yellow!18} \method{} (Ours)                            & \yes & \yes & \yes & \yes & \yes \\
\bottomrule
\end{tabular}
\end{table}

\noindent\textbf{AD benchmarks.}
Curated benchmarks run every detector on many series and report that the ranking
changes from series to series \citep{paparrizos2022tsb,liu2024elephant,zhou2026mtsbench}. Widely used
datasets fall to trivial baselines \citep{wu2023current}, and a random score beats a
real detector when one flagged point inside an anomaly counts as detecting all of it
\citep{kim2022towards}. We therefore adopt a curated benchmark
\citep{liu2024elephant} and change neither its series nor its measure
\citep{paparrizos2022volume}, which scores over the whole anomaly rather than one
point.

\noindent\textbf{Detector selection for AD.}
One line needs a corpus of past evaluations, either to find datasets a new
series looks like \citep{zhao2021automating,zhao2022toward} or to train a classifier
that predicts the winner \citep{sylligardos2023choose,boniol2024adecimo,abdelmaksoud2026ramses}. Another needs no corpus and ranks detectors by an internal
measure read from the scores alone, which never checks a confirmed
anomaly \citep{marques2015internal,goix2016how}. Neither is reliable without labels, since
internal measures alone are not significantly different from random
\citep{ma2023need} and over half of the automated pipelines fail to beat a random
choice \citep{liu2025tsbautoad}. We instead read each detector's accuracy on a
confirmed anomaly, and Table~\ref{tab:families} places
\method{} against them.

\noindent\textbf{Combining detectors.}
Outlier ensembles average normalized scores or ranks instead of picking a single
detector \citep{aggarwal2015theoretical,zimek2014ensembles}, with diversity from feature
subsets picked by hand \citep{lazarevic2005feature}. An average gives every detector the same small
share, whether it fails or leads. A second line therefore prunes detectors first
\citep{rayana2016less,zhao2019lscp}. However, there is no way to tell a right detector
from a confidently wrong one without labels, indicating that both can only ask
which detectors agree with each other.
We instead ask how accurate each detector is on the
confirmed event, and we use that answer to weight the pool rather than prune it, with no
parameter to fit on a window as short as one event.

\noindent\textbf{Semi-supervised AD.}
Semi-supervised detectors train on normal data with a few labeled anomalies
\citep{ruff2020deep,pang2019deep,pang2021deep,lai2024openset}. Active discovery asks an analyst to label the
points a model is unsure about \citep{das2016incorporating}. Both spend the labels inside a
detector, which needs training data, a training loop, and a detector built to take them.
However, a deployed monitor rarely offers these, since its detectors come already
fixed inside a running pipeline. We instead \textit{keep every detector fixed} and use the
label to decide how much to trust each, placing \method{} with methods that adjust
rather than retrain a model \citep{lee2026distpfn,martinez2026conformal}.

\section{Problem Setup}
\label{sec:setup}

\noindent\textbf{Task.}
We are given one TS of length $T$ and a pool of $N$ detectors. Detector $i$ produces a
score vector $s_i \in \mathbb{R}^T$, where a \textit{higher} score marks a \textit{more anomalous}
time point. The binary ground truth $y \in \{0,1\}^T$ marks the anomalous points, and a single
anomaly event is defined as a stretch of consecutive ones. The task is to turn the $N$ score vectors
into the single score vector $z$. We judge a score vector with a ranking measure
$A(y, s) \in [0,1]$, which returns how well $s$ ranks the anomalous points above the
normal ones, and we read it on a set of time points by restricting $y$ and $s$ to that set.

\begin{wrapfigure}{r}{0.46\textwidth}
\vspace{-12pt}
\centering
\includegraphics[width=0.44\textwidth]{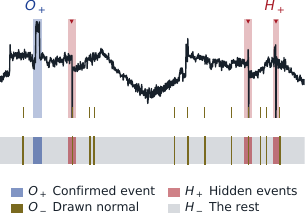}
\vspace{-4pt}
\caption{Notation for the four sets.}
\label{fig:sets}
\vspace{-20pt}
\end{wrapfigure}

\noindent\textbf{Confirmed events.}
An operator confirms $k$ of the events, with those time points written $O_+$, and every remaining event
stays hidden from the method and is used only for evaluation. We write $v_I$ for the entries a vector $v$ holds at a set $I$ of time points. Note that the pool is fixed before any label
arrives and \textit{no detector is ever retrained}. Throughout, $k$ is one unless we say
otherwise, which is where a deployment starts.

\noindent\textbf{Observed and hidden parts.}
As a ranking measure reads the gap between 1) what
a detector gives an anomaly and 2) what it gives a normal point, we cannot rank detectors
on the confirmed anomalous points $O_+$ alone. We therefore pair $O_+$ with an equal
number of time points drawn uniformly and treated as normal, denoted as $O_-$, and we
write $O$ for that union and $H$ for everything else as
\begin{equation}
\label{eq:obs}
O \;=\; O_+ \cup O_-, \qquad
H \;=\; H_+ \cup H_- \;=\; \{1,\dots,T\} \setminus O,
\end{equation}
where $H_+$ and $H_-$ hold the anomalous and the normal time points the method never
reads, as shown in Figure~\ref{fig:sets}. While randomly drawn points in $O_-$ may include a
few anomalies, we find the method effective under this noise, since anomalies occupy a
small share of a series (Appendix~\ref{app:neg}).

\noindent\textbf{Oracle gap.}
We write $A_H(s) = A(y_H, s_H)$ for the accuracy of a score vector on the hidden part,
which is where every number we report is measured. A practitioner
without labels has no reason to prefer one detector over another, which makes the
natural reference the accuracy of a detector drawn at random,
$A_{\text{rand}} = \frac{1}{N}\sum_i A_H(s_i)$. The best detector for this series,
$A_{\text{max}} = \max_i A_H(s_i)$, is unknown at deployment time and serves as an upper
bound. We report where a combined score vector $z$ lands between the two,
\begin{equation}
\label{eq:recovery}
\rho(z) \;=\; \frac{A_H(z) - A_{\text{rand}}}{A_{\text{max}} - A_{\text{rand}}}.
\end{equation}
A value of zero, i.e., $A_H(z) = A_{\text{rand}}$,
indicates that $z$ is no better than a detector drawn at random, while a
value of one, i.e., $A_H(z) = A_{\text{max}}$, indicates that it matches the best
detector for the TS.

\section{Method}
\label{sec:method}

\begin{figure}[t]
\centering
\includegraphics[width=0.98\textwidth]{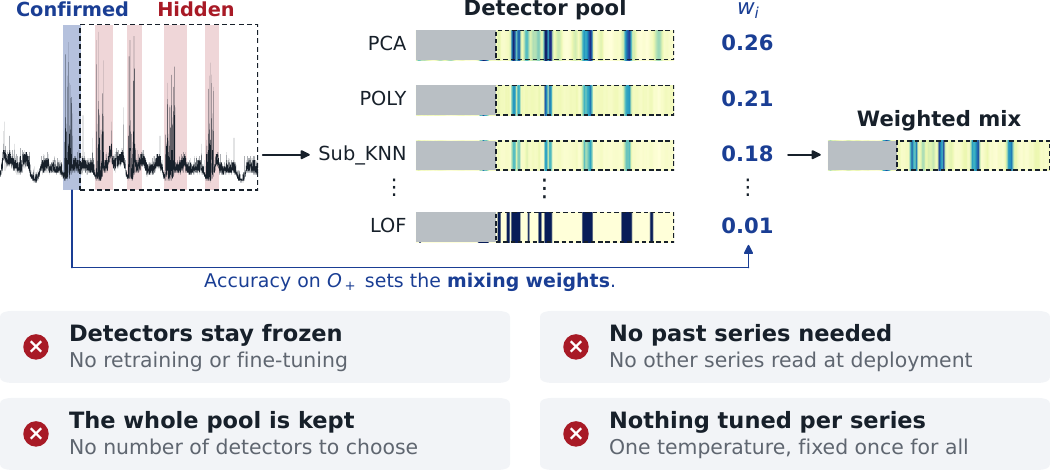}
\caption{\textbf{Overall framework of \method{}.} Each detector is scored on the confirmed event alone, while red events stay hidden, and a softmax over those accuracies gives the weight $w_i$ that the mix follows. The boxes list what the pipeline never does, and \kshot{} runs it unchanged on more events.}
\label{fig:concept}
\end{figure}

\subsection{\method{}: One Event to a Weighted Mix}
\label{sec:algo}

\noindent\textbf{Scoring the pool.}
We standardize every score vector to zero mean and unit variance, which lets vectors
from different detectors be added without one scale dominating the sum. We then score
each detector on the observed part with the same measure we report,
\begin{equation}
\label{eq:acc}
a_i \;=\; A_O(s_i) \;=\; A\big(y_O,\; s_{i,O}\big), \qquad i = 1,\dots,N,
\end{equation}
where $y_O$ and $s_{i,O}$ read the label vector and the score vector at the
$|O| = 2|O_+|$ observed time points, which means $A_O$ and $A_H$ differ in nothing but
the part of the series they are read on.

\noindent\textbf{Turning accuracy into weights.}
We map the accuracies to weights with a softmax at temperature $\tau$ and combine the
score vectors with those weights,
\begin{equation}
\label{eq:weights}
w_i \;=\; \frac{\exp\!\big((a_i - \max_j a_j)/\tau\big)}
               {\sum_{l}\exp\!\big((a_l - \max_j a_j)/\tau\big)},
\qquad
z \;=\; \sum_i w_i\, s_i .
\end{equation}
Subtracting $\max_j a_j$ only keeps the exponentials in range. A large $\tau$ spreads the
weight across all detectors, while a small one keeps it on the few that scored best, and
Appendix~\ref{app:theory} shows that $w$ is the entropy-regularized choice between the two. We choose $\tau$ once on held-out series, i.e., a deployment needs
no corpus of its own (Appendix~\ref{app:tau}). As shown in Figure~\ref{fig:temp}, it stays
effective over a wide range. Every decision is then made on $z$ rather than on any single detector. Note that the
label neither trains a detector nor tunes $\tau$, and costs one pass over $O$.

\subsection{From \method{} to \kshot{}}
\label{sec:nshot}

\noindent\textbf{Events arrive in order.}
A running monitor keeps raising alarms, and an operator keeps confirming them. \textit{Nothing in \method{} is tied to the first one.} Write $C_1,\dots,C_k$ for the windows confirmed so
far, in the order they arrive. The observed part grows with them,
\begin{equation}
\label{eq:obs-k}
O_{+,k} \;=\; \bigcup_{j\le k} C_j, \qquad O_k \;=\; O_{+,k} \cup O_{-,k},
\end{equation}
where $O_{-,k}$ again holds an equal number of points drawn uniformly and treated as
normal, and $H_k$ is the hidden part that is left, which shrinks as $k$ grows.

\noindent The procedure of Section~\ref{sec:algo} is unchanged, except that it now reads $O_k$
in place of $O$,
\begin{equation}
\label{eq:weights-k}
a^{(k)}_i \;=\; A_{O_k}(s_i),
\qquad
w^{(k)}_i \;=\; \frac{\exp\!\big((a^{(k)}_i - \max_j a^{(k)}_j)/\tau\big)}
                     {\sum_{l}\exp\!\big((a^{(k)}_l - \max_j a^{(k)}_j)/\tau\big)},
\qquad
z^{(k)} \;=\; \sum_i w^{(k)}_i\, s_i .
\end{equation}
\method{} is the special case $k=1$ of \kshot{}, i.e., the least supervision a deployment
can hold, and Algorithm~\ref{alg:oneshot} gives its pseudocode.
Appendix~\ref{app:kshot} visualizes \kshot{} on one series.

\begin{algorithm}[t]
\caption{Pseudocode of \method{}.}
\label{alg:oneshot}
\begin{algorithmic}[1]
\Require Score vectors $s_1,\dots,s_N \in \mathbb{R}^T$,
         Confirmed anomalous points $O_+ \subset \{1,\dots,T\}$

\Ensure Combined score $z \in \mathbb{R}^T$, Alarm threshold $\theta \in \mathbb{R}$
\State $s_i \gets (s_i - \mathrm{mean}(s_i)) / \mathrm{std}(s_i)$ \quad for every $i$
       \Comment{One scale for all detectors}
\State Draw $|O_+|$ time points uniformly and label them normal as $O_-$
\State $O \gets O_+ \cup O_-$
\For{$i = 1$ to $N$}
  \State $a_i \gets A_O(s_i) \in [0,1]$
         \Comment{Accuracy on confirmed events}
\EndFor
\State $w_i \gets \exp((a_i - \max_j a_j)/\tau) \,/\, \sum_l \exp((a_l - \max_j a_j)/\tau)$
       \Comment{Softmax at $\tau$, $w \in \mathbb{R}^N$}
\State $z \gets \sum_i w_i\, s_i \in \mathbb{R}^T$
\State $\theta \gets \arg\max_{t} F_1(y_O,\, z_O > t)$
       \Comment{Alarm when $z$ exceeds $\theta$}
\State \Return $z,\ \theta$
\end{algorithmic}
\end{algorithm}

\subsection{Setting the Alarm Threshold}
\label{sec:thresh}

\noindent To decide whether a time point is anomalous, a monitor needs not only the
combined score $z$ but also a threshold $\theta$. A common choice is to alarm on a fixed fraction of the
points, which fixes the number of alarms regardless of how many anomalies the series
holds. The confirmed event lets us
measure $\theta$ instead, since every cut $t$ turns the mix into the alarms $z_O > t$
that we can score against the labels on $O$. We score those alarms with $F_1$ and keep
the cut that scores highest,
\begin{equation}
\label{eq:thresh}
\theta \;=\; \arg\max_{t} \; F_1\big(y_O,\; z_O > t\big),
\end{equation}
and the threshold that separates $O_+$ from $O_-$ is the one we deploy. Note that this \textit{costs nothing beyond the label we already
have}, and Appendix~\ref{app:deploy} reports what it gains.

\section{Experiments}
\label{sec:exp}

\subsection{Experimental Setup}
\label{sec:expsetup}

\noindent\textbf{Dataset.}
We use both the univariate and the multivariate collection of TSB-AD~\citep{liu2024elephant}, a
curated AD benchmark of \ntotal{} and 180 series. Following the protocol of prior benchmark
studies~\citep{liu2024elephant,paparrizos2022tsb}, we report the two separately rather than as
one number, since the two are scored by different pools. The main text reads the univariate
collection and Appendix~\ref{app:mv} reads the multivariate one. We take the released score
vectors as they are, which keeps every detector at the hyperparameters the benchmark tuned. We
drop a detector whose scores are constant or near-constant, since those make the oracle look
perfect for reasons unrelated to detection, leaving 10 to 15 candidates per series.

\noindent\textbf{Detectors \& metrics.}
The pool holds three kinds of detector, statistical methods, deep detectors and time series
foundation models, and Appendix~\ref{app:pools} lists every one of them with the counts. Every
comparison uses four ranking measures, AUC-PR, AUC-ROC, VUS-PR and
VUS-ROC~\citep{paparrizos2022volume}. Note that $\rho$ and every temperature we choose are read
under VUS-PR alone.

\noindent\textbf{Eligible series.}
A series needs a second event to evaluate on once the first is confirmed, and a
confirmed event long enough to score. Of the \ntotal{}
univariate series, \nsingle{} hold one event, \ndegenerate{} carry constant labels,
and \nshort{} have only events too short to score, which leaves \neligible{}. A further \ndropped{} drop out once the protocol
runs, since the event a trial draws can still be too short to score, which leaves
\nseries{}.

\noindent\textbf{Protocol.}
We draw the confirmed events uniformly at random, which keeps a single event from deciding the
result, and Table~\ref{tab:first} also reads the earliest events in time. We pair them with an
equal number of drawn normal points, and we keep a series only when $H_+$ still holds a whole
event. Note that every method
sees the same observed part and is scored on the same hidden part.

\subsection{Main Experiments}
\label{sec:mainexp}
\label{sec:main}

\IfFileExists{tables/detectors.tex}{%
\noindent\textbf{Comparison with single detector.}
A practitioner who cannot choose a detector fixes one for every series, and we first
ask whether one confirmed event is worth more than any fixed choice. To answer it, we score
every detector in the pool on every series and compare each one against \method{}.
As shown in Table~\ref{tab:detectors}, \method{} reaches a higher accuracy than the
strongest single detector, and no single detector wins on more than a quarter of the
series. Note that several detectors land below a randomly drawn one, which means
fixing the wrong detector costs more than not choosing at all.

\begin{table}[!t]
\centering
\caption{\textbf{Comparison with single detector.} Each of the first rows fixes one
detector for every series, which is what a practitioner does without a way to choose.
Wins is how often a row scores best and $\rho$ is the share of the oracle gap it
closes, both read under VUS-PR. \method{} outperforms every single detector under all
four measures, even the best detector on the confirmed window.}
\label{tab:detectors}
\small
\setlength{\tabcolsep}{3.4pt}
\begin{tabular}{l|c|cccc|c}
\toprule
Method & Wins & AUC-PR & AUC-ROC & VUS-PR & VUS-ROC & $\rho$ \\
\midrule
Sub\_KNN \citep{ramaswamy2000efficient} & 17.1 & 0.521 & 0.854 & 0.623 & 0.875 & 60.5\% \\
PCA \citep{shyu2003novel} & 23.7 & 0.494 & 0.786 & 0.537 & 0.815 & 38.9\% \\
Sub\_PCA \citep{shyu2003novel} & 8.6 & 0.504 & 0.786 & 0.531 & 0.819 & 37.2\% \\
IForest \citep{liu2008isolation} & 9.2 & 0.468 & 0.785 & 0.501 & 0.829 & 29.8\% \\
Sub\_IForest \citep{liu2008isolation} & 0.0 & 0.460 & 0.787 & 0.491 & 0.833 & 27.1\% \\
POLY \citep{liu2024elephant} & 5.3 & 0.406 & 0.762 & 0.482 & 0.797 & 24.8\% \\
Sub\_HBOS \citep{goldstein2012histogram} & 4.6 & 0.433 & 0.768 & 0.461 & 0.794 & 19.5\% \\
Sub\_LOF \citep{breunig2000lof} & 13.2 & 0.232 & 0.679 & 0.352 & 0.726 & -7.9\% \\
SR \citep{ren2019time} & 2.0 & 0.320 & 0.738 & 0.351 & 0.826 & -8.3\% \\
FFT \citep{liu2024elephant} & 4.6 & 0.264 & 0.655 & 0.335 & 0.707 & -12.2\% \\
MatrixProfile \citep{yeh2016matrix} & 5.9 & 0.300 & 0.693 & 0.335 & 0.744 & -12.3\% \\
KMeansAD\_U \citep{liu2024elephant} & 5.9 & 0.288 & 0.704 & 0.315 & 0.722 & -17.4\% \\
LOF \citep{breunig2000lof} & 0.0 & 0.142 & 0.570 & 0.171 & 0.669 & -53.9\% \\
\midrule
Score average of all & -- & 0.551 & \sd{0.915} & 0.596 & \sd{0.942} & 53.6\% \\
Best on labeled window & -- & \sd{0.597} & 0.911 & \sd{0.672} & 0.935 & 73.0\% \\
\rowcolor{yellow!18} \method{} & -- & \hi{0.642} & \hi{0.941} & \hi{0.710} & \hi{0.961} & \hi{82.8\%} \\
\bottomrule
\end{tabular}

\vspace{1.1\baselineskip}
\caption{\textbf{Comparison with ensemble methods.} Each row combines the same detector
pool on the same series, and the last column is a signed-rank test of that row
against \method{}. \method{} closes the largest share of the gap and leads both
precision-recall measures.}
\label{tab:main}
\small
\setlength{\tabcolsep}{6.2pt}
\begin{tabular}{l|cccc|cc}
\toprule
Method & AUC-PR & AUC-ROC & VUS-PR & VUS-ROC & $\rho$ & $p$ \\
\midrule
\multicolumn{7}{l}{\textit{No labels}} \\
\quad Random detector & 0.337 & 0.733 & 0.384 & 0.784 & 0.0\% & $<$0.001 \\
\quad Score average & 0.551 & 0.915 & 0.596 & 0.942 & 53.6\% & $<$0.001 \\
\quad Rank average & 0.371 & 0.861 & 0.413 & 0.899 & 7.5\% & $<$0.001 \\
\midrule
\multicolumn{7}{l}{\textit{One confirmed event}} \\
\quad Best on labeled window & 0.597 & 0.911 & 0.672 & 0.935 & 73.0\% & $<$0.001 \\
\quad Top-2 average & 0.622 & \sd{0.932} & 0.702 & 0.954 & 80.5\% & 0.135 \\
\quad Top-3 average & 0.628 & \hi{0.941} & \sd{0.704} & 0.958 & \sd{81.2\%} & 0.158 \\
\quad Top-10 average & 0.598 & 0.931 & 0.649 & 0.954 & 67.3\% & $<$0.001 \\
\quad Linear weighting & \sd{0.633} & \hi{0.941} & 0.689 & \hi{0.962} & 77.3\% & $<$0.001 \\
\rowcolor{yellow!18} \quad \method{} & \hi{0.642} & \hi{0.941} & \hi{0.710} & \sd{0.961} & \hi{82.8\%} & -- \\
\midrule
\rowcolor{black!6} Oracle (Best detector for each) & 0.701 & 0.960 & 0.779 & 0.975 & 100.0\% & $<$0.001 \\
\bottomrule
\end{tabular}

\end{table}

}{}

\noindent\textbf{Comparison with ensemble methods.}
To ask how well a single confirmed event can be spent, we compare \method{} with every rule
that reads the same event, together with the ensembles that use no labels. As shown in
Table~\ref{tab:main}, \method{} closes more of the gap to the best detector than any alternative.
We attribute this gain to the soft weighting itself, since the linear weighting reads the same
accuracies and still falls behind. While the top-$K$ averages come closest, no fixed $K$ is the
best rule on more than a third of the series (Appendix~\ref{app:weights}). Note that \method{}
scores higher than every row on more than half of the series under all four measures.

\subsection{Ablation Study}
\label{sec:analysis}
\leavevmode\vspace{-\baselineskip}

\begin{wrapfigure}{r}{0.32\textwidth}
\vspace{-1.2\baselineskip}
\centering
\includegraphics[width=\linewidth]{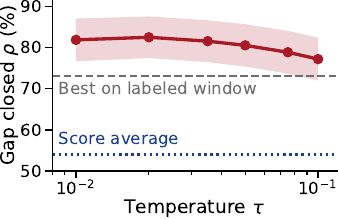}
\caption{Sensitivity to $\tau$.
}
\label{fig:temp}
\vspace{-\baselineskip}
\end{wrapfigure}
\noindent\textbf{Effect of temperature.}
To ask how carefully the temperature has to be set, we sweep it over \ntemps{} values
on the same series. As shown in
Figure~\ref{fig:temp}, the curve sits above the best detector on the window everywhere
and far above the score average. The curve is flat over half an
order of magnitude, where a signed-rank test separates no temperature from the best
one. A practitioner therefore picks a
temperature once rather than tuning one per series.

\IfFileExists{tables/pool.tex}{%
\noindent\textbf{Effect of detector diversity.}
To measure the effectiveness of diverse detectors, we compare three pools of growing
diversity on univariate and multivariate series. The third pool (Window lengths) adds no
new method and runs the same detectors at nine window lengths, as an operator would
without knowing the right one (Appendix~\ref{app:pools}). As shown in Tables~\ref{tab:pool} and~\ref{tab:mv},
\method{} leads every pool on both collections, and on univariate series it gains most
on the widest one. This indicates that diversity pays only where weights can use it.
}{}

\IfFileExists{tables/mv.tex}{%
\IfFileExists{figs/fig_pool.pdf}{%
\begin{wrapfigure}{r}{0.36\textwidth}
\vspace{-1.2\baselineskip}
\centering
\includegraphics[width=\linewidth]{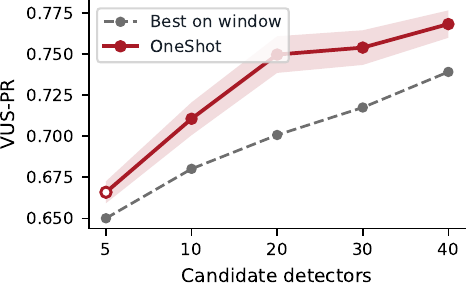}
\caption{Effect of detector count
}
\label{fig:pool}
\vspace{-\baselineskip}
\end{wrapfigure}
\noindent\textbf{Effect of detector count.}
While Table~\ref{tab:pool} widens the detector pool by model categories,
Figure~\ref{fig:pool} shows the effect of pool size alone, by sampling a fixed number
of candidates at random from the widest pool. The figure demonstrates that a larger pool
raises accuracy, and that \method{} leads the window winner at every size. The
shaded band is one standard error of the per-series difference between the two lines.
}{}

\begin{table}[t]
\begin{minipage}[t]{0.485\textwidth}
\centering
\caption{Detector diversity -- Univariate}
\label{tab:pool}
\footnotesize
\setlength{\tabcolsep}{1pt}
\begin{tabular*}{\linewidth}{@{\extracolsep{\fill}}l|ccc@{}}
\toprule
& Score & Best on & \\
\multirow{-2}{*}{Pool (\# Detectors)} & average & window & \multirow{-2}{*}{\method{}} \\
\midrule
Statistical (18) & 0.598 & 0.675 & \hi{0.712} \\
+ DL (29) & 0.612 & 0.672 & \hi{0.715} \\
+ Window lengths (77) & 0.610 & 0.682 & \hi{0.732} \\
\bottomrule
\end{tabular*}

\end{minipage}
\hfill
\begin{minipage}[t]{0.485\textwidth}
\centering
\caption{Detector diversity -- Multivariate}
\label{tab:mv}
\footnotesize
\setlength{\tabcolsep}{1pt}
\begin{tabular*}{\linewidth}{@{\extracolsep{\fill}}l|ccc@{}}
\toprule
& Score & Best on & \\
\multirow{-2}{*}{Pool (\# Detectors)} & average & window & \multirow{-2}{*}{\method{}} \\
\midrule
Statistical (11) & 0.280 & 0.382 & \hi{0.401} \\
+ DL (22) & 0.581 & 0.614 & \hi{0.633} \\
+ Window lengths (31) & 0.531 & 0.563 & \hi{0.591} \\
\bottomrule
\end{tabular*}

\end{minipage}
\end{table}

\IfFileExists{tables/grow_c1.tex}{\begin{table}[t]
\begin{minipage}[t]{0.38\textwidth}
\centering
\caption{\textbf{Effect of which event is confirmed.} Rows give $k$,
uniform above and earliest in time below.}
\label{tab:first}
\scriptsize
\setlength{\tabcolsep}{6.27pt}
\begin{tabular}{l|ccc}
\toprule
\multicolumn{1}{c|}{Confirmed} & Score & Best on & \\
\multicolumn{1}{c|}{events} & average & window & \multirow{-2}{*}{\kshot{}} \\
\midrule
\multicolumn{4}{l}{\textit{Uniformly chosen}} \\
$k{=}1$ (138) & 0.599 & 0.677 & \hi{0.710} \\
$k{=}2$ (116) & 0.581 & 0.668 & \hi{0.703} \\
$k{=}3$ (70) & 0.456 & 0.558 & \hi{0.609} \\
\midrule
\multicolumn{4}{l}{\textit{Earliest in time}} \\
$k{=}1$ (138) & 0.610 & 0.673 & \hi{0.703} \\
$k{=}2$ (116) & 0.592 & 0.665 & \hi{0.704} \\
$k{=}3$ (70) & 0.472 & 0.562 & \hi{0.601} \\
\bottomrule
\end{tabular}

\end{minipage}
\hfill
\begin{minipage}[t]{0.58\textwidth}
\centering
\caption{\textbf{Effect of more labels.} Rows give the events $m$ a series
carries and columns $k$, where $k{=}0$ leaves the weights flat. Cells are
the gap closed $\rho$ (\%).}
\label{tab:grow}
\scriptsize
\setlength{\tabcolsep}{4.12pt}
\begin{tabular}{c|c|c|cccccccc}
\toprule
& & & \multicolumn{8}{c}{w/ \kshot{}} \\
\cmidrule(l){4-11}
\multirow{-2.5}{*}{$m$} & \multirow{-2.5}{*}{Series} & \multirow{-2.5}{*}{$k{=}0$} & $1$ & $2$ & $3$ & $4$ & $5$ & $6$ & $7$ & $8$ \\
\midrule
2 & 105 & 67.4 & 86.4 & \hi{88.4} & -- & -- & -- & -- & -- & -- \\
3 & 57 & 58.9 & 80.4 & 86.9 & \hi{89.3} & -- & -- & -- & -- & -- \\
4 & 48 & 54.2 & 75.8 & 77.4 & 77.2 & \hi{83.0} & -- & -- & -- & -- \\
5 & 37 & 67.4 & 81.2 & 88.0 & 90.3 & 91.7 & \hi{95.7} & -- & -- & -- \\
6 & 38 & 65.9 & 79.6 & 89.0 & 89.9 & 89.9 & 93.7 & \hi{95.8} & -- & -- \\
7 & 37 & 61.8 & 74.7 & 82.3 & 82.1 & 83.8 & 87.3 & 88.0 & \hi{88.9} & -- \\
8 & 29 & 60.8 & 68.4 & 84.4 & 83.1 & 89.8 & 89.8 & 93.0 & 93.1 & \hi{93.1} \\
\bottomrule
\end{tabular}

\end{minipage}
\end{table}
}{}
}{}

\IfFileExists{tables/first.tex}{%
\noindent\textbf{Effect of which event is confirmed.}
Our main protocol draws the confirmed event at random, while a deployment confirms the
earliest events first, as the budget sweep of Table~\ref{tab:grow} does. To ask whether the
order matters, we compare the two, holding the series fixed so only the choice changes. As shown in Table~\ref{tab:first}, the two orders
reach almost the same accuracy at every $k$, and \kshot{} leads the window winner in both.
Furthermore, the margin over the window winner widens with $k$ in both.

}{}

\noindent\textbf{Effect of more labels.}
To ask what a further label is worth, we fix the scoring set
and vary only how many of the earliest events an operator confirms. As shown in
Table~\ref{tab:grow}, the first confirmed event lifts the gap closed by 8 to 22 points
over no label at all, and every further event adds to it.

\subsection{Analysis}
\label{sec:moreanalysis}
\leavevmode\vspace{-\baselineskip}

\IfFileExists{tables/noisebig.tex}{%
\begin{wraptable}{r}{0.50\textwidth}
\vspace{-1.1\baselineskip}
\centering
\caption{\textbf{Types of imperfect labels.} The event spans $[s, e)$ and
$L = e - s$ is its length.}
\label{tab:noisekinds}
\footnotesize
\setlength{\tabcolsep}{4.35pt}
\begin{tabular}{@{}l|c|c@{}}
\toprule
& & Event \\
\multirow{-2}{*}{Label} & \multirow{-2}{*}{Confirmed window} & covered \\
\midrule
Exact                  & $[s,\; e)$                    & 100\% \\
Half of it             & $[s,\; s + L/2)$              & 50\%  \\
Quarter of it          & $[s,\; s + L/4)$              & 25\%  \\
Shifted by half        & $[s \pm L/2,\; e \pm L/2)$    & 50\%  \\
Shifted by its length  & $[s \pm L,\; e \pm L)$        & 0\%   \\
Normal window          & $[a,\; a + L)$                & 0\%   \\
\bottomrule
\end{tabular}
\vspace{-\baselineskip}
\end{wraptable}

\noindent\textbf{Robustness to imperfect labels.}
In practice, a confirmation an operator gives is rarely exact (i.e., it can be late,
partial, or placed on a window that holds no anomaly). To validate the effectiveness of \method{} under
such labels, we corrupt the confirmed event in the ways of
Table~\ref{tab:noisekinds}.
As shown in Table~\ref{tab:noisebig}, \method{} leads in every row, and its margin
widens as the label degrades. Furthermore, we define \textit{Retention}, where each row is
divided by the exact-label row of the same column, to read how much of its own accuracy
a rule keeps as the label gets worse. \method{} outperforms the window winner on
this measure as well, in every corrupted row.

\begin{table}[t]
\centering
\caption{\textbf{Robustness to imperfect labels: Performance under each corruption.} Each row corrupts the confirmed event in the way Table~\ref{tab:noisekinds} gives. Retention$^{*}$ divides a row by the exact-label row of the same column, which reads how much a rule keeps as the label degrades.}
\label{tab:noisebig}
\label{tab:noise}
\footnotesize
\begin{tabular}{l|cc|c|cc|c|c}
\toprule
\multirow{3.5}{*}{Confirmed event} & \multicolumn{3}{c|}{VUS-PR} & \multicolumn{3}{c|}{Retention$^{*}$} & \multirow{3.5}{*}{$p$} \\
\cmidrule(lr){2-4}\cmidrule(lr){5-7}
 & \cthead{Best on\\window} & \cthead{\method{}} & \cthead{Imp (\%)} & \cthead{Best on\\window} & \cthead{\method{}} & \cthead{Imp (\%)} & \\
\midrule
Exact & 0.568 & \hi{0.699} & \hi{+23.0} & 100\% & 100\% & +0.0 & $<$0.001 \\
Half of it & 0.519 & \hi{0.692} & \hi{+33.3} & 91\% & \hi{99\%} & \hi{+8.4} & $<$0.001 \\
Quarter of it & 0.476 & \hi{0.672} & \hi{+41.0} & 84\% & \hi{96\%} & \hi{+14.7} & $<$0.001 \\
Shifted by half & 0.543 & \hi{0.668} & \hi{+22.9} & 96\% & \hi{96\%} & +0.0 & $<$0.001 \\
Shifted by its length & 0.419 & \hi{0.554} & \hi{+32.1} & 74\% & \hi{79\%} & \hi{+7.5} & $<$0.001 \\
Normal window & 0.314 & \hi{0.468} & \hi{+49.0} & 55\% & \hi{67\%} & \hi{+21.2} & $<$0.001 \\
\bottomrule
\end{tabular}
\end{table}

\begin{figure}[t]
\begin{minipage}[b]{0.36\textwidth}
\centering
\includegraphics[width=\linewidth]{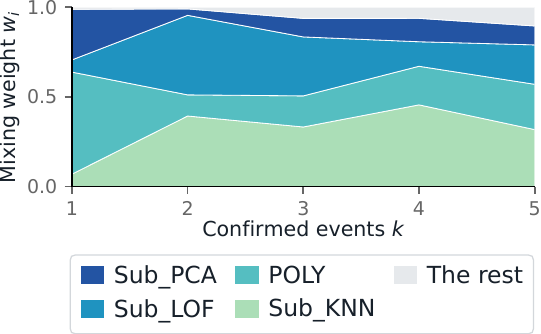}
\caption{\textbf{Soft weights by event.} The mixing weights on a single series
change after each confirmed event, as the events arrive in order.}
\label{fig:wflow}
\end{minipage}
\hfill
\begin{minipage}[b]{0.565\textwidth}
\centering
\includegraphics[width=\linewidth]{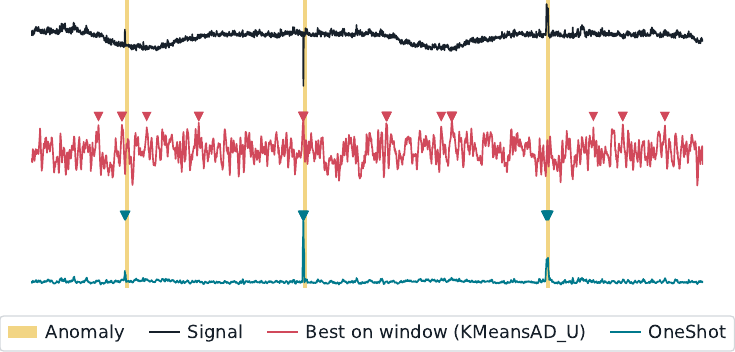}
\caption{\textbf{AD results: Best on window vs. \method{}.} Triangles mark where
each rule raises an alarm.}
\label{fig:cases}
\end{minipage}
\end{figure}
}{}

\IfFileExists{figs/fig_wflow.pdf}{%
\noindent\textbf{Evolution of soft weights.}
To see what the extra labels do to the mixture itself, we follow one series and read
the weights again after each confirmation. As shown in Figure~\ref{fig:wflow}, the weight moves
from one detector to others as events arrive, and it never gathers on a single one.

}{}

\IfFileExists{figs/fig_cases.pdf}{%
\noindent\textbf{AD visualization.}
To see where the best detector on the confirmed window and \method{} raise alarms, we
plot the signal and both scores on one series. As shown in Figure~\ref{fig:cases}, the window winner
raises false alarms across the series, while \method{} alarms accurately on the hidden
events.

}{}

\IfFileExists{tables/cost.tex}{%
\noindent\textbf{Efficiency analysis.}
The cost of \method{} splits in two, 1) the pool a deployment already runs and 2) the
work we add on top of it. We time each step of \method{} and every detector in the
pool on \ncost{} series, in CPU seconds. As shown in Table~\ref{tab:cost}, \method{} adds
\costms{}~ms, less than the median detector of the pool takes alone, to a pipeline that
already spends \costpool{}~s on the scores it reads. Note that the detectors are independent and a
deployment can run them in parallel. Furthermore, the pool is not a cost \method{}
introduces, since every row of Table~\ref{tab:main} reads the same pool.

}{}

\IfFileExists{tables/agree.tex}{%
\noindent\textbf{Source of the gain: Disagreement between the two rankings.}
Since the confirmed window covers only a single event, the detector ranking it produces need
not match the ranking on the rest of the series. The window winner is therefore a
single choice made from an uncertain ranking, and a soft weight replaces that choice with a
weighted combination. To validate this, we sort the series into quarters by how closely 1)
\textit{the ranking on the window} agrees with 2) \textit{the ranking on the hidden part},
since \method{} should gain nothing where the window already ranks them correctly. As
shown in Table~\ref{tab:agree}, the gain is large where the two rankings disagree and vanishes
where they agree, indicating that \method{} helps when the window is wrong and does no harm
when it is right.

\begin{table}[t]
\begin{minipage}[t]{0.475\textwidth}
\centering
\caption{\textbf{Efficiency analysis.} \method{} adds a negligible thousandth of the pool it reads.}
\label{tab:cost}
\small
\setlength{\tabcolsep}{3pt}
\begin{tabular}{l|c}
\toprule
What is timed & CPU seconds \\
\midrule
\multicolumn{2}{l}{\textit{Detector}} \\
\quad Fastest detector (SR) & 0.0098 \\
\quad Median detector & 0.0482 \\
\quad Slowest detector (EIF) & 44.8237 \\
\rowcolor{black!6} Pool of 17 detectors & 44.9453 \\
\midrule
\multicolumn{2}{l}{\textit{\method{}}} \\
+ Standardizing the pool & 0.0015 \\
+ Scoring the pool on the window & 0.0055 \\
+ Softmax and weighted sum & 0.0003 \\
+ Threshold search on the window & 0.0374 \\
\rowcolor{yellow!18} Total & 0.0447 \\
\bottomrule
\end{tabular}

\end{minipage}
\hfill
\begin{minipage}[t]{0.505\textwidth}
\centering
\caption{\textbf{Rank agreement and gain.} Series are split into quarters by how well the ranking of the detectors on the window agrees with their ranking on the hidden part. The gain is large on every quarter whose rankings disagree, and small elsewhere.}
\label{tab:agree}
\small
\setlength{\tabcolsep}{2.3pt}
\begin{tabular}{l|ccc|c}
\toprule
\multicolumn{1}{c|}{Agreement on$^{*}$} & Best on & & & \\
\multicolumn{1}{c|}{the window} & window & \multirow{-2}{*}{\method{}} & \multirow{-2}{*}{Gain} & \multirow{-2}{*}{Imp.} \\
\midrule
1Q (Agree) & 0.924 & \hi{0.925} & +0.001 & +0.2\% \\
2Q & 0.719 & \hi{0.766} & +0.047 & +6.5\% \\
3Q & 0.628 & \hi{0.692} & +0.063 & +10.1\% \\
4Q (Disagree) & 0.483 & \hi{0.527} & +0.045 & +9.2\% \\
\bottomrule
\end{tabular}
\\[2pt]{\footnotesize $^{*}$\,Spearman correlation between the two rankings.}

\end{minipage}
\end{table}

}{}

\IfFileExists{figs/fig_rank.pdf}{%
\begin{wrapfigure}{r}{0.44\textwidth}
\vspace{-1.2\baselineskip}
\centering
\includegraphics[width=\linewidth]{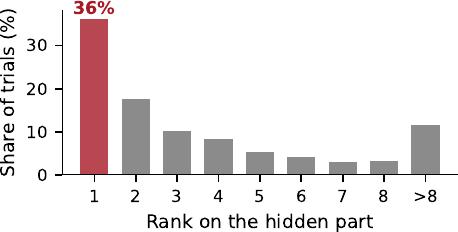}
\caption{Rank of the window winner.
}
\label{fig:rank}
\vspace{-\baselineskip}
\end{wrapfigure}
\noindent\textbf{Reliability of the window winner.}
A rule that commits to one detector is right exactly as often as the window's top detector is
the best one on the rest of the series. We record the rank that the window's top detector holds
on the hidden part, once for each confirmed event on each series. As shown in
Figure~\ref{fig:rank}, the window's top detector is the true best on only a third of the
trials, and often lands far below. A soft weight spreads over the detectors the window scores
well, indicating that a wrong ranking costs it less than it costs a single pick.

}{}

\noindent\textbf{Other analysis.}
The appendix carries the experiments below, each with its tables and discussion.
\begin{itemize}[leftmargin=1.2em,itemsep=0pt,topsep=2pt,parsep=0pt]
\item \textit{Multivariate series} (App.~\ref{app:mv}): \method{} leads a smaller, harder pool on every experiment.
\item \textit{Results by detector} (App.~\ref{app:break}): The most frequent winner still loses on three quarters of them.
\item \textit{Wins on individual series} (App.~\ref{app:metrics}): \method{} leads on more than half of them, every measure.
\item \textit{Setting the alarm threshold} (App.~\ref{app:deploy}): A threshold read off the window beats a fixed rate on F1.
\item \textit{How many detectors the weights blend} (App.~\ref{app:weights}): No fixed top-$K$ matches the effective count.
\item \textit{How the temperature is chosen} (App.~\ref{app:tau}): One value for every series, never read on that series.
\item \textit{What the detector pools hold} (App.~\ref{app:pools}): Every detector of every pool, with counts and citations.
\item \textit{Score preprocessing} (App.~\ref{app:prep}): What the near-constant filter drops, and what it costs the pool.
\item \textit{Sensitivity to the anomaly rate} (App.~\ref{app:rate}): The gain shrinks as the anomalies grow denser.
\item \textit{Protocol details} (App.~\ref{app:protocol}): How the normal half is drawn, and what a larger one buys.
\end{itemize}

\section{Conclusion}
\label{sec:conclusion}

We introduce \method{}, a label-efficient way to weight a detector pool from one confirmed
event: the scores on that event become a softmax weight rather than a choice of one detector.
Across \nseries{} univariate and \nmvmetricseries{} multivariate series, it beats the window
winner and both unsupervised ensembles, and matches the best top-$K$ average without choosing
$K$. Even so, the gain is bounded by the pool, since weighting cannot add a detector. Our confirmations
come from a benchmark, not an operator, whose errors Table~\ref{tab:noisebig} simulates, and
each series is weighted on its own. Both are next steps.

\section*{AI Use Statement}
We used generative AI tools as a general-purpose writing and coding aid: polishing and
tightening the wording of the manuscript, and drafting experiment, plotting and analysis
scripts. We did not use them to generate the research idea or to design \method{}. Every result in this paper comes from the released code run on public
benchmarks. All AI-assisted text was read and edited by the authors, and all AI-assisted code was
reviewed and re-run by the authors before its output entered the paper. We take responsibility
for the final content of this work, including text, claims, and artifacts produced with the aid
of generative AI.

\section*{Ethics Statement}
This work uses only public benchmark archives of time series and their released anomaly labels.
It involves no human subjects, no private or personally identifiable data, and no user study. The
confirmed anomaly that \method{} consumes is read from those labels, indicating that the operator in our
setting is simulated rather than real, and no operator was observed or recorded. We see no harm
from the method beyond the usual risk that any deployed monitor carries when it misses an
anomaly. The weights it produces are per series and can be inspected before the monitor is
trusted.

\section*{Reproducibility Statement}
The code for \method{} and for the evaluation protocol is available at
\url{https://github.com/seunghan96/one_shot}, and the benchmarks it runs on are public.
The detector pool and its hyperparameters are listed in Appendix~\ref{app:pools}, the datasets in
Appendix~\ref{app:data}, the metrics in Appendix~\ref{app:met}, the score preprocessing in
Appendix~\ref{app:prep}, and the temperature in Appendix~\ref{app:tau}. Every reported number is
an average over repeated draws of the confirmed event, and each series is seeded by a checksum of
its name, i.e., a rerun returns the same draws. Appendix~\ref{app:protocol} states how the observed
window and the normal half are drawn, and Appendix~\ref{app:scope} states the scope of our
claims and points to the result that answers each concern.

\bibliography{refs}
\bibliographystyle{iclr2027_conference}

\clearpage
\appendix

\section{Experimental Details}
\label{app:details}

\subsection{Metrics}
\label{app:met}

\noindent\textbf{Ranking measures.}
We read every comparison under four ranking measures, i.e., AUC-PR, AUC-ROC, VUS-PR and
VUS-ROC~\citep{paparrizos2022volume}. All four come from TSB-AD's own implementation at a sliding
window of 100 points. A ranking measure asks only whether anomalous points score above normal
ones. It needs no threshold, and \textit{the ordering is precisely what a weighted combination
changes}. Note that VUS-PR alone decides the weights and the temperature, and \textit{the
remaining three measures are read only after every choice has been made}.

\noindent\textbf{The four measures.}
Write $y \in \{0,1\}^T$ for the labels and $z \in \mathbb{R}^T$ for the score a method
produces. Sweeping a threshold $t$ over $z$ traces a true positive rate, a false
positive rate and a precision, and the first two measures are the areas these curves
enclose,
\begin{equation}
\text{AUC-ROC} \;=\; \int_{0}^{1} \text{TPR}\,d\,\text{FPR},
\qquad
\text{AUC-PR} \;=\; \int_{0}^{1} \text{Prec}\,d\,\text{Rec}.
\end{equation}
A point-wise reading is harsh on a detector that flags an event a few points early or
late. Every labeled anomaly range is therefore widened by $\ell$ points on each side
\citep{paparrizos2022volume}, the two areas are read on the widened labels, and they are
averaged over every buffer up to a length $L$,
\begin{equation}
\text{VUS-ROC} \;=\; \frac{1}{L}\int_{0}^{L} \text{AUC-ROC}(\ell)\,d\ell,
\qquad
\text{VUS-PR} \;=\; \frac{1}{L}\int_{0}^{L} \text{AUC-PR}(\ell)\,d\ell,
\end{equation}
where $\ell = 0$ recovers the point-wise reading and $L$ is the sliding window of 100
points. All four measures read only the ranking of $z$, and $\rho$ uses VUS-PR, the benchmark's
most reliable measure.

\vspace{10pt}
\subsection{Choice of the Temperature}
\label{app:tau}

\noindent\textbf{A single temperature for every series.}
The temperature $\tau$ is the only quantity \method{} does not read from the confirmed event,
and it therefore needs a rule of its own. The \ntemps{} candidates are \taulist{}, spaced
roughly geometrically from a near-argmax weight to a near-uniform one. We hold out one series at a time and apply to it the
candidate with the highest mean VUS-PR over all the others, i.e., \textit{no series takes part
in the choice that scores it}. The \nseries{} choices nevertheless agree on a single value,
$\tau = \taupick{}$. A deployment inherits that value without tuning anything itself, and
every $k$ in this paper is scored with it. A monitor that holds no corpus of its own can
take \taupick{} as a default, since Figure~\ref{fig:temp} sits above the window winner at
every candidate and a signed-rank test separates no temperature from the best one.

\noindent\textbf{Leave-one-out against a split in half.}
The obvious alternative splits the series into two halves and lets each half choose the
temperature for the other. As shown in Table~\ref{tab:tauproto}, one half does best at the
smallest candidate and the other at a much larger one. Each half therefore receives the value
that suits it worst, and the result falls below every fixed temperature. A choice that sees
more series suffers less from this effect, and leave-one-out sees all of them and matches the
best fixed temperature.

\begin{table}[h]
\centering
\caption{\textbf{How the temperature is chosen.} Every row reads the same runs and differs only in how much of the rest is used to choose $\tau$. No series ever takes part in the choice that scores it.}
\label{tab:tauproto}
\small
\begin{tabular}{l|c|cc}
\toprule
Choice of $\tau$ & Values picked & VUS-PR & $\rho$ \\
\midrule
\multicolumn{4}{l}{\textit{Chosen on the other series}} \\
\quad 2-fold & 0.01, 0.035 & 0.701 & 80.5\% \\
\quad 5-fold & 0.01, 0.02, 0.035 & 0.702 & 80.6\% \\
\quad 10-fold & 0.01, 0.02 & \sd{0.707} & 81.9\% \\
\quad Leave-one-out (Ours) & 0.02 & \hi{0.710} & 82.8\% \\
\midrule
\multicolumn{4}{l}{\textit{Fixed in advance}} \\
\quad $\tau = 0.01$ & -- & 0.709 & 82.5\% \\
\quad $\tau = 0.02$ & -- & 0.710 & 82.8\% \\
\quad $\tau = 0.05$ & -- & 0.703 & 80.8\% \\
\bottomrule
\end{tabular}
\end{table}

\clearpage
\section{Anomaly Detectors}
\label{app:pools}

\IfFileExists{tables/pools.tex}{%
\noindent\textbf{Construction of detector pools.}
Every detector scores every series once and does so before any label arrives, which means the
pool is fixed before \method{} sees anything. The first pool, \textit{Statistical} in
Table~\ref{tab:pool}, holds the statistical detectors TSB-AD provides. The second, \textit{Deep
and foundation (DL)}, adds every deep detector and time series foundation model our environment
supports. The third, \textit{+ Window lengths}, mostly varies a hyperparameter rather than
adding new methods. Since a practitioner cannot know the right window length in advance, it
keeps \nwindet{} of the detectors at all \nwinlen{} lengths, \winlist{} points. Those detectors
are \windets{}. Note that this window is the detector's own input length rather than the confirmed window,
i.e., \method{} never chooses it and only reads the scores the pool produces. The detectors in each of the three blocks
are listed in Table~\ref{tab:pools} and Table~\ref{tab:poolsmv}, for the univariate and the
multivariate collection respectively.

\begin{table}[h]
\centering
\caption{Univariate AD methods.}
\label{tab:pools}
\small
\begin{tabular}{@{}m{0.145\textwidth}|>{\centering\arraybackslash}m{0.045\textwidth}|m{0.70\textwidth}@{}}
\toprule
Block & Pool & Detectors added \\
\midrule
Statistical & 18 & CBLOF~\citep{he2003discovering}, COPOD~\citep{li2020copod}, EIF~\citep{hariri2021extended}, FFT~\citep{liu2024elephant}, HBOS~\citep{goldstein2012histogram}, IForest~\citep{liu2008isolation}, KMeansAD\_U~\citep{liu2024elephant}, KNN~\citep{ramaswamy2000efficient}, LOF~\citep{breunig2000lof}, MatrixProfile~\citep{yeh2016matrix}, PCA~\citep{shyu2003novel}, POLY~\citep{liu2024elephant}, SR~\citep{ren2019time}, Sub\_HBOS~\citep{liu2024elephant}, Sub\_IForest~\citep{liu2024elephant}, Sub\_KNN~\citep{liu2024elephant}, Sub\_LOF~\citep{liu2024elephant}, Sub\_PCA~\citep{liu2024elephant} \\
\midrule
Deep and foundation (DL) & 29 & AnomalyTransformer~\citep{xu2022anomaly}, CNN~\citep{liu2024elephant}, FITS~\citep{xu2024fits}, LSTMAD~\citep{malhotra2015long}, MOMENT\_ZS~\citep{goswami2024moment}, OFA~\citep{zhou2023one}, OmniAnomaly~\citep{su2019robust}, PatchTST~\citep{nie2023time}, TimesNet~\citep{wu2023timesnet}, TranAD~\citep{tuli2022tranad}, USAD~\citep{audibert2020usad} \\
\midrule
Window lengths & 77 & HSF~\citep{liu2024elephant} (9 window lengths), HSF\_Causal~\citep{liu2024elephant} (9 window lengths), HSF\_U~\citep{liu2024elephant} (9 window lengths), PaAno\_PAI~\citep{liu2024elephant} (9 window lengths), Left\_STAMPi~\citep{liu2024elephant}, M2N2~\citep{liu2024elephant}, MCD~\citep{rousseeuw1999fast}, MMPAD~\citep{liu2024elephant}, RobustPCA~\citep{candes2011robust}, StreamVAE~\citep{liu2024elephant}, Sub\_MCD~\citep{liu2024elephant}, Time\_RCD~\citep{liu2024elephant} \\
\bottomrule
\end{tabular}
\end{table}

\IfFileExists{tables/pools_mv.tex}{\begin{table}[h]
\centering
\caption{Multivariate AD methods.}
\label{tab:poolsmv}
\small
\begin{tabular}{@{}m{0.145\textwidth}|>{\centering\arraybackslash}m{0.045\textwidth}|m{0.70\textwidth}@{}}
\toprule
Block & Pool & Detectors added \\
\midrule
Statistical & 11 & CBLOF~\citep{he2003discovering}, COPOD~\citep{li2020copod}, EIF~\citep{hariri2021extended}, HBOS~\citep{goldstein2012histogram}, IForest~\citep{liu2008isolation}, KMeansAD~\citep{liu2024elephant}, KNN~\citep{ramaswamy2000efficient}, LOF~\citep{breunig2000lof}, OCSVM~\citep{scholkopf2001estimating}, PCA~\citep{shyu2003novel}, RobustPCA~\citep{candes2011robust} \\
\midrule
Deep and foundation (DL) & 22 & AnomalyTransformer~\citep{xu2022anomaly}, CNN~\citep{liu2024elephant}, FITS~\citep{xu2024fits}, LSTMAD~\citep{malhotra2015long}, MOMENT\_ZS~\citep{goswami2024moment}, OFA~\citep{zhou2023one}, OmniAnomaly~\citep{su2019robust}, PatchTST~\citep{nie2023time}, TimesNet~\citep{wu2023timesnet}, TranAD~\citep{tuli2022tranad}, USAD~\citep{audibert2020usad} \\
\midrule
Window lengths & 31 & PaAno\_PAI~\citep{liu2024elephant} (3 window lengths), Left\_STAMPi~\citep{liu2024elephant}, M2N2~\citep{liu2024elephant}, MCD~\citep{rousseeuw1999fast}, MMPAD~\citep{liu2024elephant}, StreamVAE~\citep{liu2024elephant} \\
\bottomrule
\end{tabular}
\end{table}
}{}
}{}

\clearpage
\section{Datasets}
\label{app:data}

\noindent\textbf{Benchmark and scores.}
All experiments run on TSB-AD~\citep{liu2024elephant}, a curated anomaly detection
benchmark that releases the score vector of every detector it provides on every series.
Every experiment reads those vectors unchanged, and \textit{no detector is ever
retrained or retuned on the series it scores}. A detector is dropped on a series
where its scores are constant or near-constant, since such a score carries no ranking
and only inflates the oracle.

\noindent\textbf{Univariate series.}
The univariate collection of TSB-AD holds \ntotal{} series drawn from \nsrcall{} sources, of
which Table~\ref{tab:datau} lists the \nsrcu{} that survive the filter below. Three conditions
remove a series before the protocol runs, i.e., 1) a single event leaves nothing to evaluate on
once it is confirmed, 2) a constant label gives no anomaly to find, and 3) an event under six
points is too short to score. These remove \nsingle{}, \ndegenerate{} and \nshort{} series,
which leaves \neligible{}. A further \ndropped{} series drop out during the protocol itself,
since the event a trial draws can again be too short to score, which leaves \nseries{}. The
series that remain carry three events and \ndetmed{} usable detectors at the median.

\noindent\textbf{Multivariate series.}
The multivariate collection of TSB-AD holds 180 series, and nine detectors score a series at
the median, against \ndetmed{} in the univariate collection. Table~\ref{tab:datamv} lists the
\nsrcmv{} sources it draws on. We apply the same eligibility rule as on the univariate series,
i.e., a series needs more than one event, a varying label and an event long enough to score,
which leaves \nmvmetricseries{} series. Note that how a detector reads the channels is a
question of its own \citep{lee2025channel}, and we change nothing about it here.
Appendix~\ref{app:mv} repeats every experiment on them.

\begin{table}[h]
\begin{minipage}[t]{0.49\textwidth}
\centering
\caption{\textbf{Univariate datasets.} The \nsrcu{} sources behind the \nseries{} univariate series. Events and Detectors are medians over the series of that source.}
\label{tab:datau}
\small
\setlength{\tabcolsep}{3pt}
\begin{tabular}{l|ccc}
\toprule
\multicolumn{1}{c|}{Source} & Series & Events & Detectors \\
\midrule
Exathlon & 30 & 3 & 13 \\
WSD & 17 & 5 & 14 \\
SMD & 15 & 3 & 14 \\
SVDB & 13 & 13 & 13 \\
OPPORTUNITY & 12 & 2 & 14 \\
NAB & 9 & 2 & 14 \\
TODS & 9 & 24 & 15 \\
IOPS & 8 & 16 & 14 \\
LTDB & 8 & 73 & 13 \\
MGAB & 8 & 4 & 14 \\
NEK & 5 & 2 & 14 \\
MITDB & 4 & 7 & 13 \\
YAHOO & 4 & 3 & 14 \\
MSL & 3 & 2 & 14 \\
SED & 2 & 19 & 14 \\
SMAP & 2 & 2 & 12 \\
CATSv2 & 1 & 5 & 14 \\
Daphnet & 1 & 6 & 14 \\
Power & 1 & 4 & 13 \\
\midrule
\textit{All} & 152 & -- & -- \\
\bottomrule
\end{tabular}

\end{minipage}
\hfill
\begin{minipage}[t]{0.49\textwidth}
\centering
\caption{\textbf{Multivariate datasets.} The \nsrcmv{} sources behind the \nmvmetricseries{} multivariate series. A source that also appears in Table~\ref{tab:datau} holds a different number of series there, since the two collections read the same data at different widths.}
\label{tab:datamv}
\small
\setlength{\tabcolsep}{3pt}
\begin{tabular}{l|ccc}
\toprule
\multicolumn{1}{c|}{Source} & Series & Events & Detectors \\
\midrule
Exathlon & 25 & 3 & 9 \\
SMD & 19 & 10 & 9 \\
LTDB & 4 & 53 & 6 \\
MSL & 4 & 2 & 9 \\
SMAP & 4 & 2 & 9 \\
SVDB & 4 & 37 & 6 \\
OPPORTUNITY & 3 & 2 & 8 \\
CATSv2 & 2 & 3 & 9 \\
GHL & 2 & 2 & 9 \\
SWaT & 2 & 5 & 8 \\
Daphnet & 1 & 6 & 9 \\
GECCO & 1 & 15 & 9 \\
Genesis & 1 & 3 & 9 \\
MITDB & 1 & 2 & 6 \\
\midrule
\textit{All} & 73 & -- & -- \\
\bottomrule
\end{tabular}

\end{minipage}
\end{table}

\clearpage
\section{Results by Detector}
\label{app:break}

\IfFileExists{tables/detail.tex}{%
A detector that leads on average need not lead on any one series, and we therefore read the
same \nseries{} series one detector at a time. As shown in Table~\ref{tab:detail}, the most
frequent winner still loses on more than three quarters of them, and every detector reaches
both first and last place.

\begin{table}[h]
\centering
\caption{\textbf{Results by detector.} Every detector of the first pool on the \nseries{} series the main tables read. Wins is how often it reaches the best VUS-PR, and the rank columns say where it lands otherwise.}
\label{tab:detail}
\small
\begin{tabular}{l|cc|ccc}
\toprule
& & & \multicolumn{3}{c}{Rank} \\
\cmidrule(lr){4-6}
\multirow{-2}{*}{Detector} & \multirow{-2}{*}{Wins (\%)} & \multirow{-2}{*}{Mean VUS-PR} & Median & Best & Worst \\
\midrule
PCA & 23.7 & 0.537 & 4 & 1 & 13 \\
Sub\_KNN & 17.1 & 0.623 & 3 & 1 & 12 \\
Sub\_LOF & 13.2 & 0.352 & 8 & 1 & 13 \\
IForest & 9.2 & 0.501 & 6 & 1 & 13 \\
Sub\_PCA & 8.6 & 0.531 & 5 & 1 & 13 \\
KMeansAD\_U & 5.9 & 0.315 & 9 & 1 & 13 \\
MatrixProfile & 5.9 & 0.335 & 9 & 1 & 13 \\
POLY & 5.3 & 0.482 & 6 & 1 & 13 \\
Sub\_HBOS & 4.6 & 0.461 & 8 & 1 & 13 \\
FFT & 4.6 & 0.335 & 9 & 1 & 13 \\
SR & 2.0 & 0.351 & 8 & 1 & 13 \\
Sub\_IForest & 0.0 & 0.491 & 6 & 2 & 13 \\
LOF & 0.0 & 0.171 & 11 & 3 & 13 \\
\bottomrule
\end{tabular}
\end{table}

}{}

\IfFileExists{figs/fig_winrate.pdf}{%
\vspace{10pt}
\clearpage
\section{Wins on Individual Series}
\label{app:metrics}

As shown in Table~\ref{tab:main}, several rows sit within a thousandth of each other under the
ROC measures, which leaves open how often each method actually wins. We therefore read the same
runs one series at a time, and each cell of Figure~\ref{fig:winrate} is the share of series on
which \method{} scores higher than that row under that measure. Every cell is above
\mbox{50\%}, indicating that \method{} wins the series-by-series comparison against every row
under every measure, and the closest of them is \winlow{}\%.

\begin{figure}[h]
\centering
\includegraphics[width=0.62\textwidth]{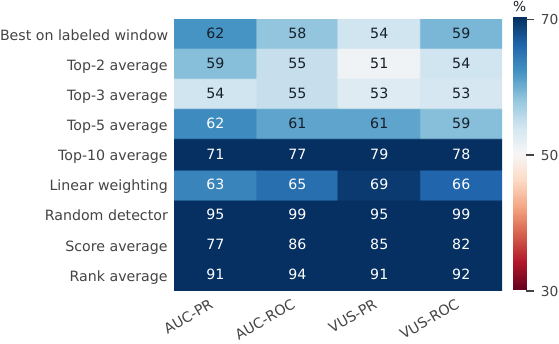}
\caption{\textbf{Win rate against each row, under each measure.} Each cell is the
share of the \nwinseries{} series where \method{} beats that row under that measure.
\textit{Every cell is above \mbox{50\%}}, i.e., \method{} wins the series-by-series
comparison against every rule under every measure.}
\label{fig:winrate}
\end{figure}

\noindent\textbf{Paired tests against every row.}
A difference between two means can come from a handful of series, and we therefore compare
every row of Table~\ref{tab:main} against \method{} one series at a time with a signed-rank
test. As shown in Table~\ref{tab:paired}, \method{} is ahead of every row that reads no
label, of the window winner, and of the linear weighting at $p<0.001$. The two rows a test
cannot separate from it are \textit{the top-2 and top-3 averages}, neither of which is a $K$ a
practitioner can know in advance.

\begin{table}[h]
\centering
\caption{\textbf{Signed-rank tests against \method{}.} Each row is the paired difference
\method{} minus that row over the \nseries{} series, under VUS-PR.}
\label{tab:paired}
\small
\begin{tabular}{l|ccc}
\toprule
Method & Median $\Delta$ & Series won & $p$ \\
\midrule
\multicolumn{4}{l}{\textit{No labels}} \\
\quad Random detector & +0.334 & 95.4\% & $<$0.001 \\
\quad Score average & +0.076 & 84.9\% & $<$0.001 \\
\quad Rank average & +0.258 & 91.4\% & $<$0.001 \\
\midrule
\multicolumn{4}{l}{\textit{One confirmed event}} \\
\quad Best on labeled window & +0.009 & 53.9\% & $<$0.001 \\
\quad Top-2 average & +0.000 & 50.7\% & 0.135 \\
\quad Top-3 average & +0.001 & 52.6\% & 0.158 \\
\quad Top-10 average & +0.035 & 78.9\% & $<$0.001 \\
\quad Linear weighting & +0.013 & 69.1\% & $<$0.001 \\
\midrule
\multicolumn{4}{l}{\textit{Upper bound}} \\
\rowcolor{black!6} Oracle (Best detector for each) & -0.028 & 19.7\% & $<$0.001 \\
\bottomrule
\end{tabular}

\end{table}
}{}
\clearpage
\section{The Effective Number of Detectors}
\label{app:weights}

\IfFileExists{figs/fig_neff.pdf}{%
\noindent\textbf{How many detectors carry weight.}
Unlike a top-$K$ average, a softmax keeps every detector and only varies how much weight each
one carries, and how concentrated that weight is changes from series to series. We measure that
concentration as the perplexity of the weight vector, which is one when a single detector holds
all the weight and the pool size when the weights are flat. As shown in Figure~\ref{fig:neff},
the effective count sits at \neffmed{} of \ndetmed{} detectors at the median, and it spreads
widely across series. A fixed top-$K$ rule has to pick one point on this axis, and no single
$K$ fits them all.
\begin{equation}
\label{eq:neff}
N_{\text{eff}}(w) \;=\; \exp\!\big(H(w)\big),
\qquad
H(w) \;=\; -\sum_{i=1}^{N} w_i \log w_i,
\end{equation}
where $w$ is the weight vector \method{} forms on a series and $N$ is the size of the
pool. Note that $N_{\text{eff}}$ carries the unit of a count, i.e., it is one when a
single detector holds all the weight, $K$ when $K$ detectors share it evenly, and $N$
when the weights are flat.

}{}

\IfFileExists{figs/fig_neff.pdf}{%
\begin{figure}[h]
\centering
\includegraphics[width=0.74\textwidth]{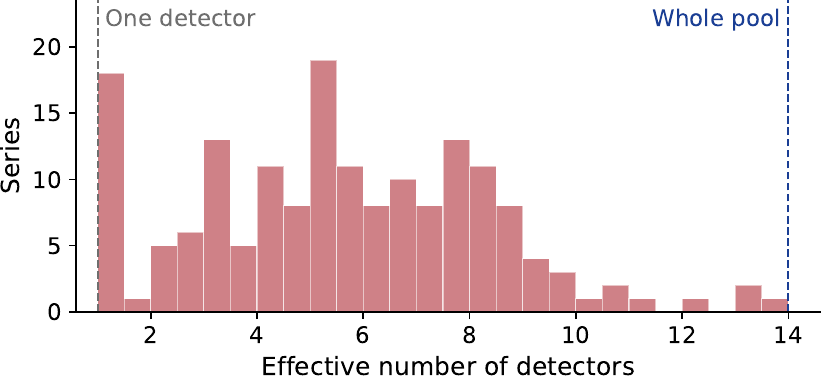}
\caption{\textbf{The number of detectors that carry weight varies from series to
series.} The histogram counts series by $N_{\text{eff}}$ of Equation~\ref{eq:neff}, and
the dashed lines mark a single detector and the whole pool.}
\label{fig:neff}
\end{figure}
}{}

\IfFileExists{figs/fig_tempneff.pdf}{%
\noindent\textbf{How the temperature spreads the weight.}
The temperature is meant to set how many detectors share the weight rather than which one wins.
To measure that, we hold the series fixed and recompute the weights at three temperatures,
reading $N_{\text{eff}}$ off each. As shown in Figure~\ref{fig:tempneff}, the median count
rises from under five to more than ten as $\tau$ grows, and reaches neither extreme. The
temperature therefore sets how widely the weight spreads, never reducing \method{} to either
end.

\begin{figure}[h]
\centering
\includegraphics[width=0.74\textwidth]{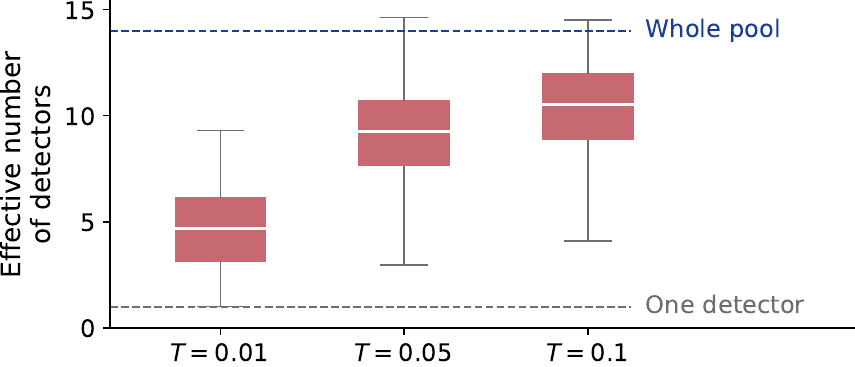}
\caption{\textbf{Effective detector count by temperature.}
$N_{\text{eff}}$ on the same series at three temperatures.}
\label{fig:tempneff}
\end{figure}
}{}

\IfFileExists{tables/neffk.tex}{%
\clearpage
\noindent\textbf{How the count moves as labels arrive.}
A second confirmation gives the weights more to read, and it could either sharpen them onto
one detector or spread them further over the pool. To ask which happens, we record
$N_{\text{eff}}$ at every budget $k$ on the series that carry enough events, holding the
temperature at the value of Appendix~\ref{app:tau}. As shown in Table~\ref{tab:neffk}, the
median count stays near its one-shot value at every budget we run, and it falls neither
toward a single detector nor toward the whole pool. \method{} therefore keeps mixing as the
budget grows rather than turning into a hard pick, which is what the accuracy of
Table~\ref{tab:grow} shows from the other side.

\begin{table}[h]
\centering
\caption{\textbf{The effective count as the budget grows.} Each row reads the series that
carry more events than it reveals, with $N_{\text{eff}}$ the perplexity of
Equation~\ref{eq:neff} at the temperature of Appendix~\ref{app:tau}.}
\label{tab:neffk}
\small
\setlength{\tabcolsep}{9pt}
\begin{tabular}{c|c|c}
\toprule
Confirmed events & Series & Median $N_{\text{eff}}$ \\
\midrule
1 & 153 & 6.1 \\
2 & 117 & 6.5 \\
3 & 72 & 6.3 \\
4 & 63 & 6.4 \\
5 & 51 & 6.9 \\
6 & 48 & 7.0 \\
7 & 47 & 5.9 \\
8 & 40 & 5.9 \\
\bottomrule
\end{tabular}

\end{table}
}{}

\IfFileExists{tables/bestk.tex}{%
\noindent\textbf{Which fixed $K$ a series wants.}
The count above is a property of the weights, while a practitioner who prefers a top-$K$
average has to name $K$ before any label arrives. To ask whether one value would serve them
all, we read the four fixed rules on every series and count which of them lands highest. As
shown in Table~\ref{tab:bestk}, the most common answer, $K=\bestkmode{}$, is the best rule on
\bestkpct{}\% of the series, and the other three take the rest in similar shares. A $K$ chosen
for each series would pay, while it has to read the hidden part to make that choice, i.e., it
is an upper bound rather than a rule a deployment can run.

\begin{table}[h]
\centering
\caption{\textbf{Which fixed $K$ is best, series by series.} The middle column counts the
series where that rule reaches the highest VUS-PR, and the last row chooses $K$ on the hidden
part, which no deployment can do and which therefore bounds the top-$K$ family rather than
joining it.}
\label{tab:bestk}
\small
\setlength{\tabcolsep}{9pt}
\begin{tabular}{l|c|c}
\toprule
Rule & Series it is best on & Mean VUS-PR \\
\midrule
Top-2 average & \hi{36.2\%} & 0.702 \\
Top-3 average & 21.7\% & 0.704 \\
Top-5 average & 23.7\% & 0.703 \\
Top-10 average & 18.4\% & 0.649 \\
\midrule
\rowcolor{yellow!18} \method{} & -- & 0.710 \\
\rowcolor{black!6} $K$ chosen for each series & -- & 0.747 \\
\bottomrule
\end{tabular}

\end{table}
}{}
\clearpage
\section{The Alarm Threshold}
\label{app:deploy}

\noindent\textbf{The same label sets the threshold.}
A deployed monitor needs a threshold as well as a detector, and practice sets it by alarming on
a fixed fraction of the points. To ask whether the confirmed event can also set the threshold,
we keep the combined score fixed and vary only how that threshold is chosen. As shown in
Table~\ref{tab:thresh}, a threshold read off the window beats both fixed fractions on F1 and on
event recall. Note that estimating the anomaly rate from the window catches the most events and
has the worst F1, which is why we read the threshold from the window directly rather than
through a rate.

\begin{table}[h]
\centering
\caption{\textbf{Choice of the alarm threshold.} Every row thresholds the same combined score.}
\label{tab:thresh}
\small
\begin{tabular}{ll|cc}
\toprule
Source & Threshold rule & F1 & Event recall \\
\midrule
\multirow{3}{*}{Practice} & Top 1\% of points & 0.243 & 0.598 \\
 & Top 5\% of points & 0.385 & 0.763 \\
 & Estimated contamination & 0.140 & \hi{0.961} \\
\midrule
\multirow{2}{*}{One confirmed event} & Highest recall at a precision floor & 0.450 & 0.794 \\
 & Highest F1 on the window & \hi{0.454} & 0.790 \\
\midrule
\rowcolor{black!6} \multirow{1}{*}{Oracle} & Highest F1 on the hidden part & 0.546 & 0.830 \\
\bottomrule
\end{tabular}
\end{table}

\vspace{10pt}
\section{Selectors That Need a Corpus}
\label{app:corpus}

\noindent\textbf{Rules that a one-shot setting cannot run.}
A second family of selectors is missing from Table~\ref{tab:main} because of the setting,
not because of the result. A meta-learner that predicts the winner from features of the series
\citep{sylligardos2023choose,boniol2024adecimo} and a framework that adapts a selection policy
over time \citep{abdelmaksoud2026ramses} both need \textit{a corpus of past evaluations}, i.e.,
many series already scored by the same pool. A deployment reading its first confirmed event
does not hold one. A semi-supervised detector trained on the few labeled intervals
\citep{lai2024openset} needs a training loop inside the detector instead. A recent benchmark
of multivariate selection reports the same gap at scale and finds that selectors reading no
label do not close it \citep{zhou2026mtsbench}, which is the gap a single confirmation
addresses here. Table~\ref{tab:families}
places every family by what it needs, and every rule we compare against reads the same
single confirmed event and nothing else.

\clearpage
\section{Results on Multivariate Series}
\label{app:mv}

\noindent\textbf{Aggregation across the channels.}
A detector on a multivariate series reads every channel and returns \textit{one score per time
point}, i.e., the aggregation happens inside the detector and the released outputs already
carry it. \method{} therefore weights a multivariate pool exactly as it weights a univariate
one, and nothing in the method depends on the number of channels. What changes is the pool, which
holds \mvndet{} detectors at the median against \medkept{} on univariate series, and every
method sits lower than its univariate counterpart.

\IfFileExists{tables/mvpool.tex}{%
\noindent\textbf{Detector diversity.}
A deployment on multivariate series runs fewer detectors, and Table~\ref{tab:mvpool} asks
whether the gain holds as that pool grows. \method{} leads at every pool size, and the margin
over the window winner grows with the pool.

\begin{table}[h]
\centering
\caption{\textbf{Pool size on multivariate series.} The same comparison as Table~\ref{tab:mv}, with the pool grown beyond the detectors the benchmark runs by default.}
\label{tab:mvpool}
\small
\begin{tabular}{l|c|c|ccc}
\toprule
Pool & Series & Detectors & Score average & Best on window & \method{} \\
\midrule
Statistical & 73 & 9 & 0.272 & 0.368 & \hi{0.382} \\
+ DL & 73 & 19 & 0.453 & 0.489 & \hi{0.503} \\
+ Window lengths & 73 & 28 & 0.482 & 0.511 & \hi{0.539} \\
\bottomrule
\end{tabular}
\end{table}

}{}

\IfFileExists{tables/mvnoise.tex}{%
\noindent\textbf{Robustness to imperfect labels.}
A confirmation on many channels is no cleaner than one on a single series, and we therefore
repeat every corruption of Table~\ref{tab:noise} here. As shown in Table~\ref{tab:mvnoise},
\method{} keeps its lead in every row, and gains most where the confirmed window holds no
anomaly at all.

\begin{table}[h]
\centering
\caption{\textbf{Imperfect labels on multivariate series.} The same corruptions as Table~\ref{tab:noisebig}, measured on the multivariate collection.}
\label{tab:mvnoise}
\small
\begin{tabular}{l|cc|c}
\toprule
Confirmed event & Best & \method{} & Imp (\%) \\
\midrule
Exact & 0.433 & \hi{0.459} & \hi{+5.9} \\
Half of it & 0.417 & \hi{0.449} & \hi{+7.5} \\
Quarter of it & 0.392 & \hi{0.437} & \hi{+11.6} \\
Shifted by half & 0.395 & \hi{0.420} & \hi{+6.3} \\
Shifted by its length & 0.317 & \hi{0.345} & \hi{+9.0} \\
Normal window & 0.296 & \hi{0.326} & \hi{+10.2} \\
\bottomrule
\end{tabular}
\end{table}

}{}

\IfFileExists{tables/mvthresh.tex}{%
\noindent\textbf{The alarm threshold.}
A multivariate monitor needs a threshold for the same reason a univariate one does. We
keep the combined score fixed and vary only how the threshold is chosen. As shown in
Table~\ref{tab:mvthresh}, reading it off the confirmed window beats both fixed fractions
on F1 and on the share of hidden events that raise an alarm.

\begin{table}[h]
\centering
\caption{\textbf{Alarm threshold on multivariate series.} The same comparison as Table~\ref{tab:thresh}, measured on the multivariate collection.}
\label{tab:mvthresh}
\small
\begin{tabular}{l|cc}
\toprule
Threshold & F1 & Events caught \\
\midrule
Top 1\% of points & 0.194 & 59.1\% \\
Top 5\% of points & 0.324 & 77.9\% \\
From the confirmed window & \hi{0.360} & \hi{87.3\%} \\
Best on the hidden part & 0.448 & 88.5\% \\
\bottomrule
\end{tabular}
\end{table}

}{}

\clearpage
\IfFileExists{tables/mvagree.tex}{%
\noindent\textbf{The same split on multivariate series.}
Section~\ref{sec:moreanalysis} traces the gain to the disagreement between the ranking
on the confirmed window and the ranking that matters. We split the \nmvdiag{}
multivariate series by the same agreement. As shown in Table~\ref{tab:mvagree}, the gain is larger
below the median than above it, while neither split separates at this sample size. Note
that the weights spread over \mvneff{} of \mvndet{} detectors here, which leaves a soft
weight less room to work with.

\begin{table}[h]
\centering
\caption{\textbf{Rank agreement and gain on multivariate series.} Series are split by how well the ranking on the confirmed window agrees with the ranking on the hidden part, as in Table~\ref{tab:agree}.}
\label{tab:mvagree}
\small
\begin{tabular}{l|c|cc}
\toprule
Agreement on the window & Series & Best on window & \method{} \\
\midrule
Below the median & 36 & 0.225 & \hi{0.261} \\
Above the median & 37 & 0.507 & \hi{0.510} \\
\bottomrule
\end{tabular}
\end{table}

}{}

\IfFileExists{tables/mvmetrics.tex}{%
\noindent\textbf{The same reading on multivariate series.}
Table~\ref{tab:mvmetrics} repeats the four headline measures on the multivariate
collection, where the pool is smaller and every method sits lower. \method{} leads
both precision-recall measures, and the top-3 average leads the two ROC measures by a
few thousandths.

\begin{table}[h]
\centering
\caption{\textbf{Multivariate series under four ranking measures.} The same reading as Table~\ref{tab:main} on the multivariate collection of TSB-AD{,} where the pool is smaller and the series are harder.}
\label{tab:mvmetrics}
\small
\begin{tabular}{l|cccc|c|c}
\toprule
Method & AUC-PR & AUC-ROC & VUS-PR & VUS-ROC & $\rho$ & $p$ \\
\midrule
\multicolumn{7}{l}{\textit{No labels}} \\
\quad Random detector & 0.211 & 0.680 & 0.211 & 0.712 & 0.0\% & $<$0.001 \\
\quad Score average & 0.262 & 0.783 & 0.272 & 0.812 & 25.0\% & $<$0.001 \\
\quad Rank average & 0.217 & 0.749 & 0.220 & 0.779 & 3.6\% & $<$0.001 \\
\midrule
\multicolumn{7}{l}{\textit{One confirmed event}} \\
\quad Best on labeled window & 0.356 & 0.804 & 0.368 & 0.832 & 64.7\% & 0.209 \\
\quad Top-2 average & \sd{0.359} & 0.817 & \sd{0.371} & 0.845 & \sd{66.0\%} & 0.011 \\
\quad Top-3 average & 0.349 & \sd{0.820} & 0.364 & \sd{0.847} & 63.2\% & 0.010 \\
\quad Top-5 average & 0.298 & 0.811 & 0.309 & 0.836 & 40.4\% & $<$0.001 \\
\quad Linear weighting & 0.318 & 0.818 & 0.330 & 0.844 & 49.3\% & $<$0.001 \\
\quad \method{} & \hi{0.369} & \hi{0.821} & \hi{0.382} & \hi{0.849} & \hi{70.4\%} & -- \\
\midrule
\rowcolor{black!6} Oracle (Best detector for each) & 0.448 & 0.871 & 0.453 & 0.890 & 100.0\% & $<$0.001 \\
\bottomrule
\end{tabular}
\end{table}

}{}

\IfFileExists{tables/mvfamily.tex}{%
\noindent\textbf{Which block of the pool the weights favor.}
The multivariate pool mixes statistical detectors with deep and foundation models, which
leaves open whether the weights prefer one kind over another. We group the detectors of the
largest pool by the blocks of Table~\ref{tab:poolsmv} and compare the share of the weight a
block receives with the share of the pool it occupies. As shown in Table~\ref{tab:mvfamily},
the statistical block receives less weight than its size on the \mvfamnseries{} series, while
the deep block and the window-length variants both receive more. Note that the weights read
accuracy on the confirmed window alone, i.e., nothing in the method knows which block a
detector belongs to.

\begin{table}[h]
\centering
\caption{\textbf{Where the weight goes on multivariate series.} The blocks are those of
Table~\ref{tab:poolsmv}, and the ratio divides the share of the weight by the share of the
pool, i.e., one means a block receives exactly its size.}
\label{tab:mvfamily}
\small
\setlength{\tabcolsep}{8pt}
\begin{tabular}{l|cc|c}
\toprule
Block & Share of the pool & Share of the weight & Ratio \\
\midrule
Statistical & 37.7\% & 24.1\% & 0.64 \\
Deep and foundation & 37.2\% & 42.7\% & 1.15 \\
Window-length variants & 25.1\% & 33.2\% & 1.32 \\
\bottomrule
\end{tabular}

\end{table}
}{}
\clearpage
\section{Protocol Details}
\label{app:protocol}

\IfFileExists{tables/cleanneg.tex}{%
\subsection{The Noise in the Drawn Normal Half}
\label{app:neg}

\noindent\textbf{Cost of the drawn normal half.}
The normal half is the set $O_-$ of Section~\ref{sec:setup}, i.e., as many time points as
the confirmed event holds, drawn uniformly and treated as normal. The draw reads the whole
series, and some of those points are therefore anomalous and reach the method with the wrong
label. To measure that cost, we read every series twice with the same confirmed event and
hidden part, changing only where $O_-$ is drawn from. As shown in Table~\ref{tab:cleanneg},
drawing it from the truly normal points instead moves the closed gap by \negdrho{} points,
which a signed-rank test does not separate from zero, although \negcontam{} of the drawn points
are anomalous at the median. We attribute the small cost to the hidden part itself, which also
holds a few anomalies among normal points and therefore resembles a noisy half more than a
clean one.

\begin{table}[h]
\centering
\caption{\textbf{What the noise in the drawn normal half costs.} Every series is read twice from the same confirmed event and the same hidden part{,} and only the source of $O_-$ changes: the whole series{,} as a deployment must{,} against the truly normal points{,} which needs the labels we avoid.}
\label{tab:cleanneg}
\small
\begin{tabular}{l|c|cc|c}
\toprule
Where $O_-$ is drawn from & Series & Best on window & \method{} & $\rho$ \\
\midrule
The whole series (ours) & 150 & 0.675 & \hi{0.712} & 84.3\% \\
The normal points (needs labels) & 150 & 0.653 & \hi{0.710} & 83.6\% \\
\midrule
Difference & & -0.022 & -0.003 & -0.7\% \\
$p$ vs.\ the first row & & 0.030 & 0.432 & \\
\bottomrule
\end{tabular}
\end{table}

\noindent\textbf{Draws of the normal half.}
Every number we report averages repeated draws rather than a single sample. A series is read
\textit{eight times}, each with a fresh confirmed event and normal half, and three times when it
exceeds \mbox{20\,000} points, which bounds the cost of the longest series. Each series is seeded
by a checksum of its name, i.e., a rerun returns the same draws and every method sees the same
observed part. Note that one draw moves a series by \sdsoft{} VUS-PR at the median under
\method{} against \sdargmax{} under the window winner, i.e., \textit{the soft weight is the
steadier of the two}.

\noindent\textbf{The size of the drawn half.}
The drawn half matches the confirmed event in size, which keeps the two classes balanced and
ties its cost to the event rather than to the series. A larger half steadies the window
reading, while a uniform draw also hands it more wrong labels. To ask which effect wins, we
repeat the protocol with halves of two and five times that size, holding the confirmed event
and the hidden part fixed. As shown in Table~\ref{tab:negsize}, the anomalous share stays flat,
while five times the points lifts both \method{} and the window winner, the latter more. Note
that every number in this paper uses the smallest half, where every method reads lowest and
the margin over the window winner is widest.

\IfFileExists{tables/negsize.tex}{%
\begin{table}[h]
\centering
\caption{\textbf{What a larger normal half buys.} Every series is read three times from the
same confirmed event and the same hidden part, and only the number of drawn normal points
changes. The last column is a signed-rank test against the first row over \negnseries{} series.}
\label{tab:negsize}
\small
\setlength{\tabcolsep}{7pt}
\begin{tabular}{l|c|cc|c|c}
\toprule
Size of $O_-$ & Anomalous in it & Best on window & \method{} & $\rho$ & $p$ vs.\ $1\times$ \\
\midrule
\rowcolor{yellow!18} $|O_+|$ (ours) & 6.38\% & 0.667 & \hi{0.713} & 83.5\% & -- \\
$2\times|O_+|$ & 6.41\% & 0.684 & \hi{0.716} & 84.4\% & 0.233 \\
$5\times|O_+|$ & 6.42\% & 0.707 & \hi{0.729} & 87.7\% & $<$0.001 \\
\bottomrule
\end{tabular}

\end{table}
}{}

}{}

\IfFileExists{tables/causal.tex}{%
\clearpage
\subsection{Temporal Order of the Observed Part}
\label{app:causal}

\noindent\textbf{Normal points drawn from the past alone.}
A deployment cannot read points that lie after the confirmed event, while a uniform draw can
place a few of them in the normal half. We therefore repeat the protocol with the draw
restricted to points \textit{before the confirmed event ends}, holding the hidden part fixed.
As shown in Table~\ref{tab:causal}, the restriction moves \method{} by \caussoft{}
($p = \causpsoft{}$) and costs the window winner \causargmax{} ($p = \causpargmax{}$), i.e., the
margin widens once the draw respects the clock. Nothing in the result therefore rests on reading
a future point.

\begin{table}[h]
\centering
\caption{\textbf{Where the normal half is drawn from.} The right column draws it only from
the points before the confirmed event ends, i.e., from what a deployment has already seen.}
\label{tab:causal}
\small
\begin{tabular}{l|cc|c}
\toprule
Method & Drawn from the whole series & Drawn from the past & $\Delta$ \\
\midrule
Best on labeled window & 0.672 & 0.656 & -0.016 \\
\rowcolor{yellow!18} \method{} & 0.710 & 0.704 & -0.006 \\
\midrule
\rowcolor{black!6} Oracle (Best detector for each) & 0.779 & 0.779 & +0.000 \\
\bottomrule
\end{tabular}

\end{table}
}{}

\IfFileExists{tables/nshot.tex}{%
\vspace{10pt}
\subsection{The Budget Sweep on a Changing Hidden Part}
\label{app:nshot}

\noindent\textbf{Difference between the two sweeps.}
We read the same values of $k$ twice, once with the hidden part shrinking as events are
revealed and once with it fixed. As shown in Table~\ref{tab:nshot}, the first reading
rises to a peak and falls back, and the oracle it is measured against falls with it. The fall is therefore a property of \textit{the scoring set} rather than of
\textit{the method}, which Table~\ref{tab:growfull} confirms on a fixed one.

\begin{table}[h]
\centering
\caption{\textbf{Confirming more than one event.} Each row reveals $k$ events and hides the rest, and $\rho$ is the fraction of the gap to the best detector that \kshot{} closes. The upper block uses every series that carries $k$ events, and the lower block repeats the sweep on the series that carry all of them.}
\label{tab:nshot}
\small
\begin{tabular}{l|c|ccc|cc}
\toprule
Confirmed events & Series & Score average & Best on window & \kshot{} & $\rho$ & $p$ \\
\midrule
\multicolumn{7}{l}{\textit{Every series that carries $k$}} \\
\quad 1 & 154 & 0.599 & 0.674 & \hi{0.711} & 82.5\% & $<$0.001 \\
\quad 2 & 122 & 0.584 & 0.667 & \hi{0.700} & 85.9\% & $<$0.001 \\
\quad 3 & 75 & 0.476 & 0.568 & \hi{0.620} & 86.2\% & $<$0.001 \\
\quad 4 & 67 & 0.499 & 0.582 & \hi{0.622} & 82.6\% & $<$0.001 \\
\quad 5 & 54 & 0.520 & 0.590 & \hi{0.647} & 88.9\% & $<$0.001 \\
\quad 6 & 53 & 0.518 & 0.574 & \hi{0.634} & 86.5\% & $<$0.001 \\
\quad 7 & 56 & 0.536 & 0.607 & \hi{0.638} & 80.1\% & 0.013 \\
\quad 8 & 52 & 0.588 & 0.660 & \hi{0.692} & 84.1\% & 0.109 \\
\midrule
\multicolumn{7}{l}{\textit{The same series throughout}} \\
\quad 1 & 37 & 0.602 & 0.702 & \hi{0.746} & 88.5\% & $<$0.001 \\
\quad 2 & 37 & 0.586 & 0.693 & \hi{0.755} & 93.9\% & $<$0.001 \\
\quad 3 & 37 & 0.579 & 0.684 & \hi{0.751} & 95.3\% & $<$0.001 \\
\quad 4 & 37 & 0.567 & 0.680 & \hi{0.741} & 95.4\% & $<$0.001 \\
\quad 5 & 37 & 0.547 & 0.657 & \hi{0.716} & 94.5\% & $<$0.001 \\
\quad 6 & 37 & 0.533 & 0.633 & \hi{0.695} & 93.2\% & $<$0.001 \\
\quad 7 & 37 & 0.504 & 0.628 & \hi{0.677} & 92.3\% & 0.002 \\
\quad 8 & 37 & 0.488 & 0.602 & \hi{0.657} & 92.6\% & 0.004 \\
\bottomrule
\end{tabular}
\end{table}

\noindent\textbf{Every $k$ on the fixed scoring set.}
Table~\ref{tab:growfull} is the full version of Table~\ref{tab:grow}. Each block
holds the series that carry that many events, and the last column tests every $k$ against
a single confirmed event read on exactly the same set of series.

\IfFileExists{tables/grow.tex}{\begin{table}[h]
\centering
\caption{\textbf{Does a second label help? (full version)} Every row reads the same hidden part, which is fixed before any event is revealed, which means every $k$ is measured on one scoring set. $\rho$ is the fraction of the gap to the best detector that \kshot{} closes, and the last column tests each $k$ against a single confirmed event on the same series. Table~\ref{tab:grow} reads the same numbers in the main text.}
\label{tab:growfull}
\small
\begin{tabular}{l|c|cc|cc}
\toprule
Confirmed events & Series & Best on window & \kshot{} & $\rho$ & $p$ vs.\ one event \\
\midrule
\multicolumn{6}{l}{\textit{Series that carry 2 events}} \\
\quad 1 & 105 & 0.715 & \hi{0.781} & 86.4\% & -- \\
\quad 2 & 105 & 0.727 & \hi{0.789} & 88.4\% & 0.615 \\
\midrule
\multicolumn{6}{l}{\textit{Series that carry 3 events}} \\
\quad 1 & 57 & 0.620 & \hi{0.709} & 80.4\% & -- \\
\quad 2 & 57 & 0.633 & \hi{0.735} & 86.9\% & $<$0.001 \\
\quad 3 & 57 & 0.677 & \hi{0.744} & 89.3\% & $<$0.001 \\
\midrule
\multicolumn{6}{l}{\textit{Series that carry 4 events}} \\
\quad 1 & 48 & 0.601 & \hi{0.691} & 75.8\% & -- \\
\quad 2 & 48 & 0.616 & \hi{0.696} & 77.4\% & 0.037 \\
\quad 3 & 48 & 0.649 & \hi{0.696} & 77.2\% & 0.512 \\
\quad 4 & 48 & 0.652 & \hi{0.717} & 83.0\% & 0.061 \\
\midrule
\multicolumn{6}{l}{\textit{Series that carry 5 events}} \\
\quad 1 & 37 & 0.613 & \hi{0.725} & 81.2\% & -- \\
\quad 2 & 37 & 0.626 & \hi{0.755} & 88.0\% & 0.003 \\
\quad 3 & 37 & 0.681 & \hi{0.768} & 90.3\% & 0.005 \\
\quad 4 & 37 & 0.691 & \hi{0.771} & 91.7\% & 0.009 \\
\quad 5 & 37 & 0.707 & \hi{0.785} & 95.7\% & 0.004 \\
\midrule
\multicolumn{6}{l}{\textit{Series that carry 6 events}} \\
\quad 1 & 38 & 0.630 & \hi{0.713} & 79.6\% & -- \\
\quad 2 & 38 & 0.637 & \hi{0.751} & 89.0\% & 0.003 \\
\quad 3 & 38 & 0.668 & \hi{0.759} & 89.9\% & $<$0.001 \\
\quad 4 & 38 & 0.679 & \hi{0.760} & 89.9\% & 0.002 \\
\quad 5 & 38 & 0.694 & \hi{0.770} & 93.7\% & $<$0.001 \\
\quad 6 & 38 & 0.712 & \hi{0.778} & 95.8\% & $<$0.001 \\
\midrule
\multicolumn{6}{l}{\textit{Series that carry 7 events}} \\
\quad 1 & 37 & 0.623 & \hi{0.711} & 74.7\% & -- \\
\quad 2 & 37 & 0.640 & \hi{0.742} & 82.3\% & 0.006 \\
\quad 3 & 37 & 0.677 & \hi{0.744} & 82.1\% & 0.011 \\
\quad 4 & 37 & 0.694 & \hi{0.753} & 83.8\% & 0.006 \\
\quad 5 & 37 & 0.717 & \hi{0.762} & 87.3\% & 0.002 \\
\quad 6 & 37 & 0.726 & \hi{0.762} & 88.0\% & 0.002 \\
\quad 7 & 37 & 0.723 & \hi{0.764} & 88.9\% & 0.001 \\
\midrule
\multicolumn{6}{l}{\textit{Series that carry 8 events}} \\
\quad 1 & 29 & 0.664 & \hi{0.693} & 68.4\% & -- \\
\quad 2 & 29 & 0.659 & \hi{0.767} & 84.4\% & $<$0.001 \\
\quad 3 & 29 & 0.689 & \hi{0.764} & 83.1\% & $<$0.001 \\
\quad 4 & 29 & 0.697 & \hi{0.790} & 89.8\% & $<$0.001 \\
\quad 5 & 29 & 0.720 & \hi{0.784} & 89.8\% & $<$0.001 \\
\quad 6 & 29 & 0.726 & \hi{0.797} & 93.0\% & $<$0.001 \\
\quad 7 & 29 & 0.732 & \hi{0.796} & 93.1\% & $<$0.001 \\
\quad 8 & 29 & 0.732 & \hi{0.795} & 93.1\% & $<$0.001 \\
\bottomrule
\end{tabular}
\end{table}
}{}
}{}

\clearpage
\section{Score Preprocessing}
\label{app:prep}

\noindent\textbf{Orientation of the scores.}
Every detector reports larger scores for more anomalous points, and we standardize each score
vector to zero mean and unit variance before mixing. The rescaling leaves the ranking of a
single detector unchanged while putting the pool on one scale, which makes a weighted sum
meaningful. A detector with the opposite convention would score below a random one on the
window and receive the smallest weight. Note that every method in Table~\ref{tab:main} reads
these same standardized vectors, i.e., the averages and the linear weighting mix exactly what
\method{} mixes, while the rank average and the window pick do not depend on the scale.

\noindent\textbf{The near-constant filter.}
We drop a score vector whose standard deviation falls below $10^{-9}$, that takes fewer than
$\max(10, 0.001T)$ distinct values on a series of length $T$, or more than half of whose entries
are not finite, and we fill the remaining non-finite entries with the median. Such a vector
carries no ranking of its own, while the measure would still read it as a tie in its favor. The
filter drops \ndrop{} of the \ntotfile{} released vectors (\droppct{}\%), leaving a median of
\medkept{} of \medraw{} detectors per series. Note that it costs the pool nothing, i.e., the
oracle is unchanged and a dropped vector is the best of its series on only \ndropbeat{} of them.

\vspace{4pt}
\noindent\textbf{The measure the weights read.}
The weights read VUS-PR on the confirmed window and the report reads VUS-PR on the hidden
part, which leaves open whether the lead belongs to the method or to the measure. The two
parts are disjoint, i.e., the weights never see a point they are later scored on. As shown in
Table~\ref{tab:main}, the same runs under three further measures leave \method{} ahead on
three of the four and just behind the best row on the fourth.
Appendix~\ref{app:metrics} repeats the comparison one series at a time and reaches the same
reading.

\vspace{10pt}
\section{Sensitivity to the Anomaly Rate}
\label{app:rate}

\noindent\textbf{Gain against the density of anomalies.}
A denser series hands the uniformly drawn normal half more points that carry the wrong label.
To ask how far that reaches, we sort the series into quarters by their anomaly rate and measure
the gain over the window winner in each. As shown in Table~\ref{tab:rate}, the gain is largest
on the sparsest quarter and shrinks with density, turning slightly negative on the densest one.
Note that a series where more than one point in eight is anomalous lies \textit{outside the
regime a monitor runs in}, and the three quarters that do match a deployment all favor
\method{}.

\begin{table}[h]
\centering
\caption{\textbf{Gain against the anomaly rate.} Series are sorted into quarters by the share
of points that are anomalous, where the gain shrinks as the series grow denser.}
\label{tab:rate}
\small
\begin{tabular}{l|c|ccc}
\toprule
Quarter & Anomaly rate & \method{} & Best on window & Gain \\
\midrule
1Q (Sparse) & 0.53\% & 0.536 & 0.419 & +0.116 \\
2Q & 3.69\% & 0.779 & 0.740 & +0.040 \\
3Q & 7.80\% & 0.692 & 0.676 & +0.016 \\
4Q (Dense) & 12.62\% & 0.835 & 0.853 & -0.018 \\
\bottomrule
\end{tabular}

\end{table}

\clearpage
\section{The \kshot{} Pipeline}
\label{app:kshot}

\noindent\textbf{The same pipeline at a larger budget.}
\kshot{} reads every event confirmed so far in place of the single one, and nothing else about
the pipeline changes. As shown in Figure~\ref{fig:kshot}, the confirmed part of the top row
holds one event and that of the bottom row holds three, while the scoring of the pool, the
softmax and the weighted mix stay as they are. The weights differ between the rows only because
the accuracies they read differ, i.e., a longer confirmed part ranks the pool more steadily.
Note that both rows use the temperature of Appendix~\ref{app:tau}, and no quantity is retuned
as $k$ grows.

\begin{figure}[h]
\centering
\includegraphics[width=0.98\textwidth]{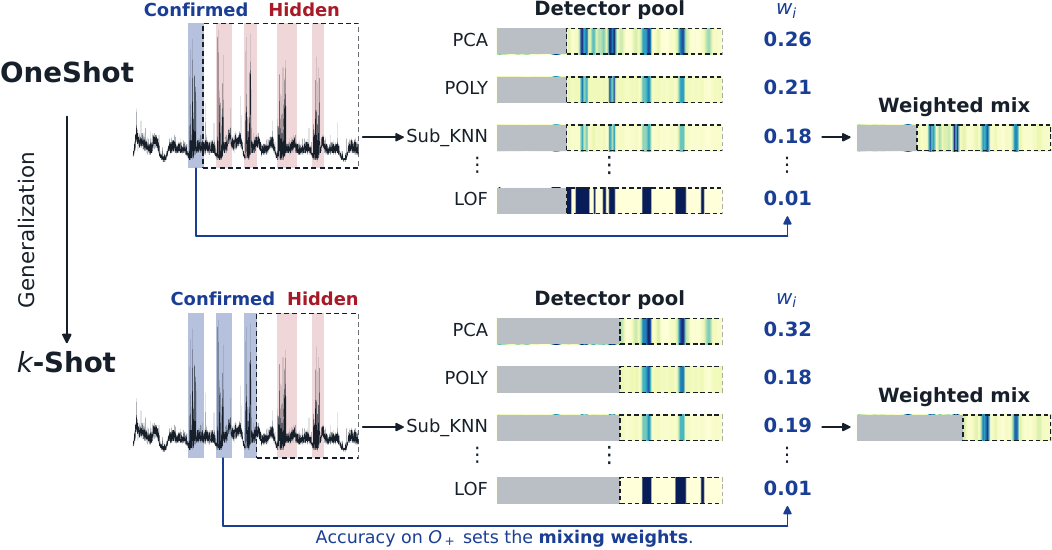}
\caption{\textbf{\method{} and \kshot{} on one series.} The top row confirms one event and the
bottom row three, while the events in red stay hidden. Both rows run the same pipeline, and
only the confirmed part grows.}
\label{fig:kshot}
\end{figure}

\clearpage
\section{Two Properties of the Softmax Weight}
\label{app:theory}

\noindent\textbf{Scope of this section.}
The two statements below are properties of the softmax rather than new theory, and we state
them because they account for two of our results. Both concern the weight $w$ that \method{}
forms from the window accuracies $a = (a_1,\dots,a_N)$ in Equation~\ref{eq:weights}. They say
nothing about the accuracy of the mixed score $z$ itself, since a ranking measure is not linear
in $w$. Note that neither needs an assumption on the detectors, the series or the measure.

\begin{proposition}[Entropy-regularized choice]
\label{prop:entreg}
Let $\Delta_N$ be the probability simplex over $N$ detectors and $H$ the entropy of
Equation~\ref{eq:neff}. For any $a \in \mathbb{R}^N$ and $\tau > 0$, the weight of
Equation~\ref{eq:weights} is the unique solution of
\begin{equation}
\label{eq:entreg}
w \;=\; \arg\max_{v \in \Delta_N} \; \sum_{i} v_i a_i + \tau H(v).
\end{equation}
As $\tau \to 0$, $w$ tends to the indicator of $\arg\max_i a_i$ when that maximizer is unique,
and as $\tau \to \infty$, $w$ tends to the uniform weight $1/N$.
\end{proposition}

\noindent\textit{Proof.} The objective is strictly concave on $\Delta_N$, since $H$ is strictly
concave and the first term is linear, which makes the maximizer unique. Its derivative in $v_i$
is $a_i - \tau(\log v_i + 1)$, which grows without bound as $v_i \to 0$, and the maximizer
therefore lies in the interior of $\Delta_N$. Setting the gradient of the Lagrangian
$\sum_i v_i a_i + \tau H(v) - \lambda(\sum_i v_i - 1)$ to zero gives
$\log v_i = (a_i - \lambda)/\tau - 1$, i.e., $v_i \propto \exp(a_i/\tau)$, which is
Equation~\ref{eq:weights}. The two limits follow from $\exp((a_i - \max_j a_j)/\tau)$, which
tends to $\mathbf{1}[a_i = \max_j a_j]$ as $\tau \to 0$ and to one for every $i$ as
$\tau \to \infty$. \hfill$\square$

\noindent\textbf{Reading of Proposition~\ref{prop:entreg}.}
The best detector on the window and the score average of Table~\ref{tab:main} are the two ends
of Equation~\ref{eq:entreg}, and \method{} sits between them at the temperature of
Appendix~\ref{app:tau}. The second term equals $\tau \log N_{\text{eff}}(w)$, i.e., the
temperature prices how many detectors share the weight, which is the count
Appendix~\ref{app:weights} measures. \method{} is therefore not a third heuristic next to the
two standard responses of Section~\ref{sec:intro}, but the choice that trades the window
accuracy against that count at a single rate. Note that the rate is fixed once for every series,
and nothing in Equation~\ref{eq:entreg} is fit to the series it weights.

\begin{proposition}[Stability of the weight]
\label{prop:stable}
For any $a, a' \in \mathbb{R}^N$ and $\tau > 0$, the weight of Equation~\ref{eq:weights}
satisfies
\begin{equation}
\label{eq:lipschitz}
\big\|w(a) - w(a')\big\|_2 \;\le\; \frac{1}{2\tau}\,\big\|a - a'\big\|_2 .
\end{equation}
In contrast, the hard pick $e_{\arg\max_i a_i}$ moves by $\sqrt{2}$ in $\ell_2$ whenever a change
of any size reverses the order of the two largest accuracies.
\end{proposition}

\noindent\textit{Proof.} The Jacobian of $w$ with respect to $a$ is
$J = \frac{1}{\tau}\big(\mathrm{diag}(w) - ww^{\top}\big)$. The matrix in brackets is the
covariance of a one-hot variable with probabilities $w$, which makes it symmetric and positive
semidefinite. Its $i$-th diagonal entry is $w_i(1 - w_i)$ and the absolute values of the rest of
its $i$-th row also sum to $w_i(1 - w_i)$, and the Gershgorin circle theorem therefore bounds
every eigenvalue by $\max_i 2w_i(1 - w_i) \le 1/2$. Hence $\|J\|_2 \le 1/(2\tau)$ at every point,
and integrating $J$ along the segment from $a'$ to $a$ gives Equation~\ref{eq:lipschitz}. Two
distinct one-hot vectors lie $\sqrt{2}$ apart, which gives the second claim. \hfill$\square$

\noindent\textbf{Reading of Proposition~\ref{prop:stable}.}
A corrupted or short label shifts the window accuracies, and Proposition~\ref{prop:stable} bounds
how far the weight can follow. At the temperature of Appendix~\ref{app:tau}, i.e.,
$\tau = \taupick{}$, a change of $0.004$ in one accuracy moves the weight by at most $0.1$ in
$\ell_2$, while the same change hands the whole weight of a hard pick to another detector when
the two best accuracies lie within it. This is consistent with the two results where our gain is
largest, i.e., the corrupted labels of Table~\ref{tab:noisebig}, where the margin widens as the
label degrades, and the series of Table~\ref{tab:agree}, whose ranking on the window disagrees
with the ranking that matters. Note that the bound holds on every series, while the room it
leaves depends on how close the leading accuracies are, which is consistent with the flat gain
Table~\ref{tab:agree} shows where the window already agrees.

\clearpage
\section{Scope of the Claims and Where Each Concern Is Answered}
\label{app:scope}

\noindent\textbf{What we do not claim.}
\method{} weights a fixed pool, and the gain it reaches is bounded by that pool, since weighting
cannot add a detector. Our comparison runs among rules that combine a fixed pool \textit{without
training}, i.e., a combiner fit on the confirmed window is outside the setting we study. The
confirmations come from benchmark labels rather than from an operator, and the corrupted labels
of Table~\ref{tab:noisebig} stand in for the ways a real one goes wrong. On series where more
than one point in eight is anomalous we report no gain, as Table~\ref{tab:rate} shows. Note that
the propositions of Appendix~\ref{app:theory} describe the weight vector alone, and they say
nothing about the accuracy of the mixed score.

\noindent\textbf{Where each concern is answered.}
A reader who doubts one of the choices below can go straight to the evidence for it, and
Table~\ref{tab:scope} lists both. Every row points to a result already in this paper rather than
to an argument, i.e., the table adds no claim of its own. We group the rows by the part of the
protocol they question, from the method itself to the labels, the pool and the cost. Note that
the multivariate collection repeats every experiment of the main text, and its own concerns
therefore share the rows below.

\begin{table}[h]
\centering
\caption{\textbf{Where each concern is answered.} Each row names a question a reader may raise
and the result that answers it, without adding any claim of its own.}
\label{tab:scope}
\small
\setlength{\tabcolsep}{6pt}
\begin{tabular}{l|l}
\toprule
Concern & Where it is answered \\
\midrule
\multicolumn{2}{l}{\textit{The method}} \\
\quad A softmax over accuracies is a standard idea & Appendix~\ref{app:theory}, Propositions 1 and 2 \\
\quad The margin over the top-$K$ averages is small & Table~\ref{tab:main}, Table~\ref{tab:bestk} \\
\quad The temperature needs labeled series & Appendix~\ref{app:tau} \\
\midrule
\multicolumn{2}{l}{\textit{The label}} \\
\quad A real operator confirms imperfectly & Table~\ref{tab:noisebig}, Table~\ref{tab:mvnoise} \\
\quad The drawn normal half carries anomalies & Appendix~\ref{app:neg} \\
\quad The protocol may read points from the future & Appendix~\ref{app:causal} \\
\quad Events may have to arrive in a fixed order & Table~\ref{tab:first} \\
\midrule
\multicolumn{2}{l}{\textit{The pool and the measure}} \\
\quad The result may depend on the pool & Table~\ref{tab:pool}, Figure~\ref{fig:pool} \\
\quad Preprocessing may favor the method & Appendix~\ref{app:prep} \\
\quad One measure may decide the ranking & Table~\ref{tab:main}, Appendix~\ref{app:metrics} \\
\midrule
\multicolumn{2}{l}{\textit{The deployment}} \\
\quad The gain may not survive dense anomalies & Table~\ref{tab:rate} \\
\quad The cost may be prohibitive & Table~\ref{tab:cost} \\
\bottomrule
\end{tabular}
\end{table}

\end{document}